\documentclass{article}
\usepackage{todonotes}

\usepackage[preprint]{corl_2026}   
\usepackage{fancyhdr}

\newcommand{\arxivheader}{\small Accepted as a conference paper at CoRL 2026}
\fancypagestyle{empty}{%
  \fancyhf{}%
  \fancyhead[L]{\arxivheader}%
  \fancyfoot[C]{\thepage}%
}

\usepackage{amsmath}
\usepackage{amsthm}
\usepackage{amssymb}

\usepackage{aliascnt}

\theoremstyle{plain}
\newtheorem{theorem}{Theorem}

\newaliascnt{proposition}{theorem}
\newtheorem{proposition}[proposition]{Proposition}
\aliascntresetthe{proposition}

\newaliascnt{lemma}{theorem}
\newtheorem{lemma}[lemma]{Lemma}
\aliascntresetthe{lemma}

\newaliascnt{corollary}{theorem}
\newtheorem{corollary}[corollary]{Corollary}
\aliascntresetthe{corollary}

\theoremstyle{definition}
\newaliascnt{definition}{theorem}

\aliascntresetthe{definition}

\theoremstyle{remark}
\newaliascnt{remark}{theorem}

\aliascntresetthe{remark}

\usepackage[ruled,vlined]{algorithm2e}
\usepackage{wrapfig}
\usepackage{placeins}
\usepackage{xspace}
\newcommand{\method}{\textsc{strips-wm}\xspace}
\usepackage{booktabs}

\usepackage{xcolor}
\usepackage{capt-of}

\usepackage{xcolor}
\usepackage{listings}
\usepackage{needspace}
\definecolor{pddlblue}{HTML}{22577A}
\definecolor{pddlrule}{HTML}{CBD5DF}
\lstdefinelanguage{LearnedPDDL}{
  sensitive=false,
  alsoletter={:,-},
  morekeywords={define,domain,:requirements,:strips,:negative-preconditions,
    :predicates,:action,:parameters,:precondition,:effect,and,not},
  morecomment=[l]{;}
}
\lstdefinestyle{paperPDDL}{
  language=LearnedPDDL,
  basicstyle=\fontsize{9}{11}\selectfont\ttfamily\color{black},
  keywordstyle=\color{pddlblue}\bfseries,
  commentstyle=\color{black!60}\itshape,
  columns=fullflexible,keepspaces=true,
  showstringspaces=false,showspaces=false,showtabs=false,
  breaklines=true,breakatwhitespace=true,breakindent=1.5em,
  tabsize=2,numbers=none,
  frame=l,rulecolor=\color{pddlrule},framerule=0.5pt,
  framesep=5pt,xleftmargin=6pt,xrightmargin=1pt,
  aboveskip=0pt,belowskip=0pt
}
\newenvironment{pddlblock}{%
  \par\addvspace{4pt}\noindent\begin{minipage}{\linewidth}%
}{%
  \end{minipage}\par
}

\usepackage[final]{microtype}
\microtypesetup{
  protrusion=true,
  expansion=true,
  tracking=true,
  kerning=true,
  spacing=true
}

\usepackage[compact]{titlesec}

\titlespacing*{\section}
{0pt}{0.4ex plus 0.1ex minus 0.1ex}{0.25ex plus 0.1ex}

\titlespacing*{\subsection}
{0pt}{0.35ex plus 0.1ex minus 0.1ex}{0.2ex plus 0.1ex}

\titlespacing*{\paragraph}
{0pt}{0.3ex plus 0.1ex minus 0.1ex}{0.5em}

\titleformat{\paragraph}[runin]
{\normalfont\normalsize\bfseries}
{}
{0pt}
{}
[]

\newcommand{\strips}{\textsc{strips}\xspace}%

 \usepackage{graphicx}
\author{
  Abhiroop Ajith\\
  Department of Robotics Engineering\\
  Worcester Polytechnic Institute\\
  United States\\
  \texttt{aajith@wpi.edu}
  \And
  Constantinos Chamzas\\
  Department of Robotics Engineering\\
  Worcester Polytechnic Institute\\
  United States\\
  \texttt{cchamzas@wpi.edu}
}

\begin{document}

\title{STRIPS-WM: Learning Grounded Propositional STRIPS-style World Models from Images}
\maketitle


\begin{abstract}
Robots performing long-horizon visual manipulation observe high-dimensional images, but
successful plans depend on action-relevant facts: what can be done now and what changes
afterward. Classical task planners exploit this structure through symbolic operators with
preconditions and effects, but obtaining such representations from raw visual experience
remains challenging. We study a visual task-planning setting in which a robot receives only
action-image transitions: the current image, executed high-level action, and resulting image.
At test time, given start and goal images, the robot must produce a sequence of high-level
actions that reaches the goal. We introduce \method, a framework for learning image-grounded
\strips-style world models directly from visual transitions. \method first induces a finite
abstract transition graph from images, then learns binary predicates and propositional
operators that explain this graph. The learned predicates are grounded back into perception through a visual classifier, enabling classical
planning directly from novel start and goal images. Experiments on visual rearrangement tasks show that \method
improves image-to-plan success over the tested visual rollout,
latent graph-search, and latent-symbolic baselines.
Execution experiments with a simulated UR10 under perception and
execution perturbations further show how grasp and placement
errors limit success on longer plans. \par\smallskip
\noindent\textbf{Project website:}
\url{https://elpislab.org/projects/strips-wm/}
\end{abstract}
\keywords{Learning Logic-Based Representations, Visual Task Planning}

\begin{center}

\vspace{-0.4em}
\begin{minipage}{0.98\linewidth}
    \centering
    \includegraphics[
        width=0.92\linewidth,
        height=0.24\textheight,
        keepaspectratio
    ]{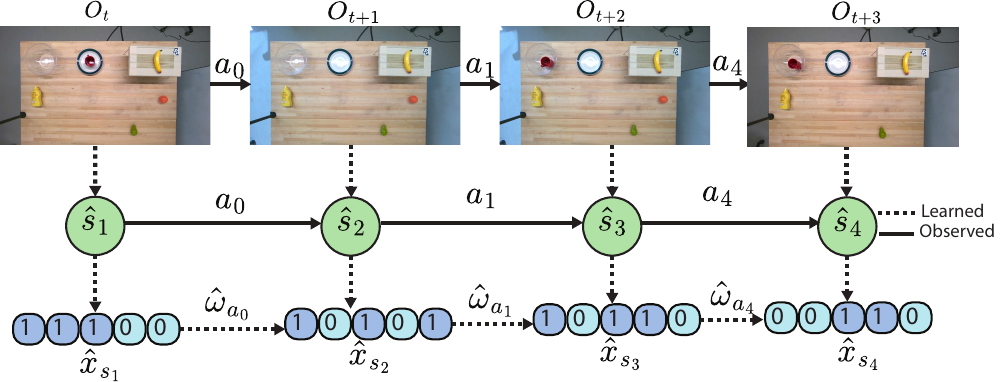}
\vspace{-0.4em}
    \captionof{figure}{\footnotesize
    Three planning spaces for the same visual task.
    \textbf{Top:} In image space, the robot observes RGB states connected by high-level
    action identifiers, \(o_t \xrightarrow{a_t} o_{t+1}\); planning here would require
    reasoning over future images.
    \textbf{Middle:} In the learned task-graph space, observations are compressed into
    abstract states \(\hat{s}\), yielding action-indexed transitions
    \(\hat{s}_t \xrightarrow{a_t} \hat{s}_{t+1}\). This enables graph search, but requires
    a flat enumeration of abstract states and edges.
    \textbf{Bottom:} In the learned \strips space, abstract states receive predicate vectors
    \(\hat{\mathbf{x}}_{\hat{s}}\), and each action identifier \(a\) receives a grounded
    operator \(\hat{\omega}_a\) with preconditions and effects. This task representation enables
    classical task planning over predicate states rather than images or enumerated graph edges.
    }
    \label{fig:intro}
\end{minipage}
\end{center}
\vspace{-0.4em}
\section{Introduction}
Robots perceive their surroundings through high-dimensional
observations such as RGB images. Yet, long-horizon manipulation tasks often admit much
simpler abstract descriptions~\cite{ghallab2004automated}. In the table-setting example
of~\autoref{fig:intro}, the robot must reach a
target arrangement by executing a sequence of high-level actions. Planning directly in
image space would require predicting long sequences of future images~\cite{finn2017foresight}. However, the pixel level is too
fine-grained for task planning: most visual variation is irrelevant to whether an action
can be applied or what it will change. What matters is the smaller set of action-relevant
conditions that hold in the scene and how actions transform them~\cite{konidaris_skills_2018}.

Classical task planning formalizes this idea through symbolic states and transition models.
In \strips planning, states are represented by symbolic facts (predicates), and actions are represented
by operators with preconditions and add/delete effects~\cite{fikes1971strips}.
Given such a model, the remaining search problem is exactly the kind of discrete planning
problem that off-the-shelf task planners can solve. This structure is illustrated at the bottom of~\autoref{fig:intro}:
instead of raw observations, \strips represents the world using symbolic states \(s\) with
predicate vectors \(\mathbf{x}_s\); instead of treating actions only as identifiers \(a\),
it represents them as operators \(\omega_a\) with preconditions and effects. The central
challenge is therefore to learn this symbolic world model from data.

We propose \method to address this challenge directly. Given only
action-indexed visual transitions $ \mathcal{D_O}=\{(o_i,a_i,o'_i)\}_{i=1}^{N},$
we learn an image-grounded \strips world model. As illustrated from the top to the bottom
of~\autoref{fig:intro}, this means learning a mapping from observations to predicate states
\(\hat{\mathbf{x}}_{\hat{s}}\), together with one operator \(\hat{\omega}_a\) for each action
identifier. We use hats to denote learned or predicted quantities.

\method learns this model in three stages, mirroring the three levels shown
in~\autoref{fig:intro}. First, an action-conditioned visual dynamics model induces an image-grounded task graph, converting visual transitions
\(o_t \xrightarrow{a_t} o_{t+1}\) into abstract transitions
\(\hat{s}_t \xrightarrow{a_t} \hat{s}_{t+1}\). Second, a CP-SAT solver assigns
binary predicate vectors \(\hat{\mathbf{x}}_{\hat{s}}\in\{0,1\}^m\) to graph nodes and
learns one grounded operator \(\hat{\omega}_a\) per action identifier, so that graph edges
are explained by \strips-style preconditions and add/delete effects. Third, a visual
predicate classifier grounds these learned predicate assignments back to images, enabling
classical planning directly from novel start and goal observations.

Our contributions are: (i) a framework for learning image-grounded propositional
\strips world models from action-indexed visual transitions, without object-centric
supervision or hand-designed predicates; (ii) a CP-SAT formulation that jointly learns
binary predicate assignments for abstract states and grounded \strips operators explaining
the observed transition graph; and (iii) an image-to-plan pipeline that grounds the learned
predicates in perception and enables classical planning from novel start and goal images.
Experiments on simulated and real rearrangement tasks show improved long-horizon
planning over visual rollout, latent graph-search, and latent-symbolic baselines.

\section{Related Work} 
\paragraph{Reconstruction-based visual planning.} Such methods plan from images by learning forward predictive models and searching over candidate action sequences, either by predicting future images~\cite{finn2017foresight,yen2020experience,ichter2019robot,paxton_visual_2019,driess_deep_2020,lippi_enabling_2023}, or by planning through imagined latent rollouts from pixels
~\cite{hafner_learning_2019,hafner2022deep,wu2022daydreamer}. These methods optimize
representations for visual prediction, often requiring reconstruction of
pixel-level observations, which can preserve details irrelevant to action applicability and
effects. \method does not reconstruct or predict images.

\paragraph{Visual planning without decoding.}
These methods avoid pixel reconstruction by learning continuous latent
representations with contrastive, predictive, or goal-conditioned objectives
~\cite{2025lecundinowm,chamzas2022-contrastive-visual-task-planning,srinivas2018universalplanningnetworks,2023lecunJEPA,kipf2020contrastivelearningstructuredworld,laskin2020curl}.
While these methods discard some irrelevant visual detail, their latent spaces lack the
logical structure of symbolic representations. Planning uses learned latent dynamics or connectivity in the latent space. 
\method learns predicate states and grounded operators, yielding a symbolic model that can be searched by classical planners. 

\paragraph{Learning symbolic abstractions.}
A complementary line of work learns symbolic abstractions for planning. This foundational line of work grounds symbols in
the preconditions and effects of high-level skills~\cite{konidaris_constructing_2014,konidaris_skills_2018}, while recent neuro-symbolic methods
learn predicates, operators, or relational transition models for planning
~\cite{silver_learning_2021,silver_learning_2023,chitnis_learning_2022,shah2024reals}.
These methods show the value of symbolic structure, but typically assume object-centric,
relational, or continuous state inputs rather than raw images. Xi et al.~\cite{xi2024visualtraces} jointly learn visual state
predictions and lifted action models from traces, using given
predicate and action signatures and observed symbolic final
states.  

Recent methods use pretrained language and vision-language models to invent and
ground predicates from visual observations~\cite{liang2025visualpredicator,athalye2026pixels,yang2025skillwrapper}. \method derives binary
predicate assignments and grounded operators from a learned
transition graph, using visual transitions labeled by opaque action
identifiers. Closest to our setting, LatPlan learns propositional
\strips from unlabeled image pairs~\cite{asailatplan}, and we include
it as our main baseline.

\section{Problem Formulation}
\label{sec:problem}
\paragraph{Observation and Action Data.}
We assume the robot observes the world through RGB images and executes actions from a
fixed finite library of high-level primitives. Let $\mathcal{O}$ denote the space of RGB
observations, and let
\mbox{ $
\mathcal{A}=\{1\ldots,L\}
$}
denote the finite set of executable action identifiers. Each opaque action identifier denotes one grounded primitive instance with fixed arguments. For example, picking the same object from two different placements corresponds to two
different identifiers. The learner receives these identifiers without object or placement metadata, preconditions, or effects. Additionally, the robot is given a dataset of visual transitions:
\[
D_{\mathcal O}
=
\{(o_i,a_i,o'_i)\}_{i=1}^{N},
\qquad
o_i,o'_i\in\mathcal{O},\;\; a_i\in\mathcal{A}.
\]
Each triplet records the observation before an action, the identifier of the executed action,
and the observation after execution. We assume no access to object identities, segmentation
masks, object poses, hand-designed predicates, symbolic state labels, or symbolic goals
during training.

\paragraph{Action-Relevant Abstraction Assumption.}
Following the skills-to-symbols perspective of Konidaris et al.~\cite{konidaris_skills_2018},
we assume that the available high-level primitives induce a finite task-level abstraction of
the underlying world. Intuitively, two physical world states belong to the same abstract
state if they are interchangeable for planning: the same high-level actions are applicable,
executing those actions leads to the same abstract outcomes (deterministic world), and the same task goals are satisfied. Low-level variations within an abstract state may change the RGB image, but they do not affect the feasibility or outcome of high-level action sequences. This assumption is what makes discrete high-level planning well-defined from visual observations.
We also assume that $\mathcal{D}_{\mathcal O}$ covers the states and transitions needed for the evaluated tasks. It need not contain every applicable action from every state: operator learning uses
selected missing state-action pairs as soft negative evidence
(\autoref{sec:operator-learning}).

\paragraph{Test-Time Planning Objective.}
At test time, the robot is given a start image $o_{\mathrm{start}}$ and a goal image
$o_{\mathrm{goal}}$. The objective is to produce a sequence:
\[
\pi = (a_1,\ldots,a_H),
\qquad a_h\in\mathcal{A},
\]
such that executing $\pi$ from the world state depicted by $o_{\mathrm{start}}$ reaches a
world state that is task-equivalent to the one depicted by $o_{\mathrm{goal}}$.

\section{Background: STRIPS Propositional Planning} \label{sec:strips_background}

We review \strips planning since our method heavily builds on its concepts.
\strips is a classical formalism for representing deterministic planning problems using
symbolic states and action models~\citep{fikes1971strips}. In the propositional setting, a vocabulary of $m$ Boolean predicates $\mathcal{P}=\{p_1,\ldots,p_m\}$ is given.

A symbolic state $\sigma$ is represented by a binary vector
$\mathbf{x_{\sigma}}\in\{0,1\}^m$, where $\mathbf{x}_{\sigma,i}=1$ means that 
$p_i$ is true in state $\sigma$, and $\mathbf{x}_{\sigma,i}=0$ means that it is false.
Each action $a$ is represented by an operator:
\[
\omega_a
=
\left\langle
\mathbf{pre}^{+}_a,
\mathbf{pre}^{-}_a,
\mathbf{add}_a,
\mathbf{del}_a
\right\rangle ,
\]
where
$\mathbf{pre}^{+}_a,\mathbf{pre}^{-}_a,\mathbf{add}_a,\mathbf{del}_a\in\{0,1\}^m$
are binary masks over the predicate vocabulary. The masks
$\mathbf{pre}^{+}_a$ and $\mathbf{pre}^{-}_a$ specify which predicates must be true
or false for action $a$ to be applicable. In state $\sigma$, action $a$ is applicable if
\[
\mathbf{pre}^{+}_a \leq \mathbf{x}_\sigma,
\qquad
\mathbf{pre}^{-}_a \leq \mathbf{1}-\mathbf{x}_\sigma,
\]
where inequalities are interpreted elementwise. If the action is applicable, it produces a successor state $\mathbf{x_{\sigma_1}}$ by setting predicates in
$\mathbf{add}_a$ to true, setting predicates in $\mathbf{del}_a$ to false, and leaving
all other predicates unchanged. Equivalently, each operator $\omega_a$ induces a
partial transition function $\tau_{\omega_a}$, defined on applicable states as:
\[
\mathbf{x}_{\sigma_1} = \tau_{\omega_a}(\mathbf{x}_{\sigma_0})
=
\mathbf{add}_a
\;\lor\;
\bigl(\mathbf{x}_{\sigma_0} \land \neg \mathbf{del}_a\bigr).
\]
The add and delete masks are disjoint, and the effects are state-independent:
whenever the preconditions of $a$ hold, the same add/delete masks are applied.
This yields a compact transition model for search over symbolic states.

A \strips planning problem consists of an initial symbolic state, a goal
condition, and a set of operators. In our setting, neither the predicate
vocabulary $\mathcal{P}$ nor the operator masks are given; both must be learned
from visual transition data.

\section{Methodology}
\label{sec:methodology}

\subsection{Image-Grounded Task Graph}

\label{sec:world-model}

The goal of this stage is to use the visual transition dataset
$D_{\mathcal O}=\{(o_i,a_i,o'_i)\}_{i=1}^{N}$ to learn an image-grounded task
graph $\hat{G}=(\hat{\mathcal S},\hat{E})$. We do this by learning a discrete
state assignment $\alpha:\mathcal O\rightarrow\hat{\mathcal S}$ that maps each
observation to an abstract-state identifier $\hat{s}=\alpha(o)$. Ideally,
observations that are interchangeable for high-level planning receive the same
state identity, even if they differ in irrelevant visual details.

\autoref{fig:worldmodel} shows the architecture used to learn this
assignment. A student encoder maps the current observation $o_t$ to a latent
representation, which is discretized with finite scalar quantization (FSQ)~\cite{mentzer2024fsq} to produce
a code $z_t$. A slowly updated teacher encoder maps the successor observation
$o_{t+1}$ to a target code $z^{\mathrm{tar}}_{t+1}$. Given $z_t$ and the opaque
action identifier $a_t$, an action-conditioned predictor is trained to predict
the successor code. This gives a JEPA-style~\cite{2023lecunJEPA} objective:
prediction happens in latent code space rather than through pixel
reconstruction.

To encourage the representation to retain action-relevant information, we
also train an inverse dynamics head~\cite{agarwal2016learningtopoke,pathak2017curiosity} that predicts the executed action
identifier from the current and successor latent codes. This auxiliary
loss discourages collapsed codes that explain the visual data but fail to
distinguish transitions produced by different actions.

After training, we apply $\alpha$ to the observations in each transition
in $D_{\mathcal O}$. This yields the learned task graph
\[
\hat G=(\hat{\mathcal S},\hat E),
\qquad
\hat E=
\left\{
\bigl(\alpha(o_i),a_i,\alpha(o'_i)\bigr)
\right\}_{i=1}^{N}.
\]
Its nodes are grounded in images, and its edges record observed
transitions under the opaque action identifiers.

\begin{figure}
    \centering
    \includegraphics[width=\linewidth]{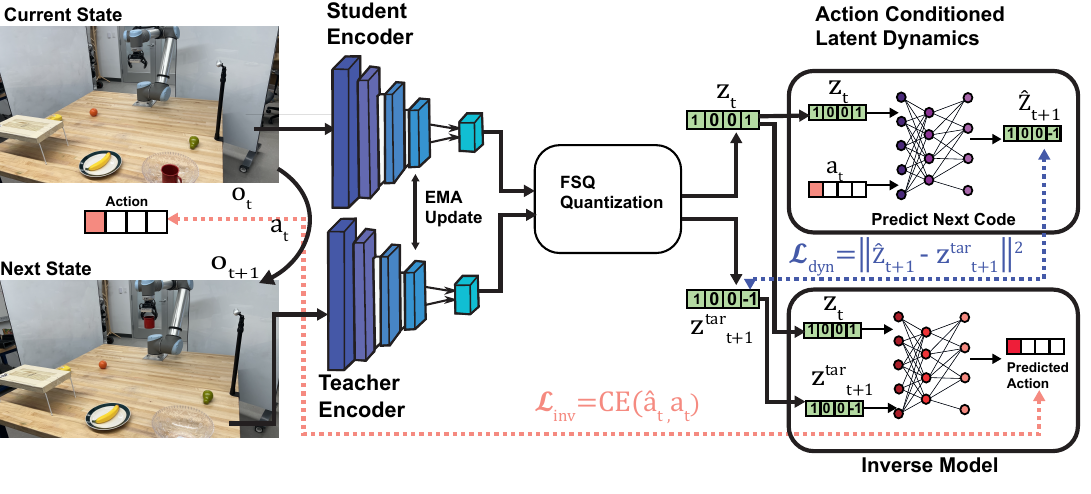}
    \vspace{-2em}

    \caption{\footnotesize\textbf{Learning the image-grounded task graph:} A student encoder maps the current observation to a quantized FSQ code, while an EMA teacher encoder provides the target successor code. An action-conditioned dynamics model predicts the next code from the current code and action identifier, and an inverse model predicts the action from current and successor codes. Unique quantized codes define visual task-graph nodes, and observed transitions define labeled graph edges. }
    \label{fig:worldmodel}
    \vspace{-1em}
\end{figure}

\subsection{Learning STRIPS Operators and Predicates from the Task Graph}
\label{sec:operator-learning}
The learned task graph $\hat G$ from~\autoref{sec:world-model}
gives us a discrete model of the visual experience, but it is
still not a symbolic planning model. Its nodes are opaque
abstract states, and its edges are labeled only by opaque action
identifiers. To perform \strips-style planning, we learn a binary
predicate vector
$\hat{\mathbf{x}}_{\hat{s}}\in\{0,1\}^m$
for each graph node $\hat{s}\in\hat{\mathcal S}$, together with
one grounded propositional operator $\hat{\omega}_a$ for each
action identifier $a\in\mathcal A$.
This assignment can give the same predicate vector to multiple
graph nodes.

We cast this as a combinatorial search over predicate assignments
and operator masks, and encode the search as a binary CP-SAT
problem~\cite{cpsatlp}. Similar in spirit to Bonet and Geffner,
who learn symbolic planning representations from state-space
graphs using a SAT-based formulation
~\cite{bonet2020learningfirstordersymbolicrepresentations},
we use graph structure as supervision. Our setting differs in
that the graph nodes are induced from raw images, and the learned
operators are grounded propositional operators rather than lifted
first-order schemas.

Observed transitions tell us which actions must be applicable and
which successor predicate vectors their effects should produce.
They do not, by themselves, identify when an action should be ruled
out: empty preconditions satisfy every observed application.
We therefore also use missing outgoing action labels as negative
evidence~\cite{2024ainetonegativelearning}.

At selected graph nodes, we provisionally treat an action as
inapplicable if its label does not appear on any outgoing edge.
This is a local closed-world assumption over action labels.
Let $\hat{\mathcal S}_{\mathrm{trust}}\subseteq\hat{\mathcal S}$
denote these selected nodes. In our experiments, this set contains
every learned node with at least one observed outgoing edge.
Missing labels at nodes outside this set are ignored.

For example, a table arrangement may have a recorded transition
for moving an object to one placement but none for moving it to
another. The unobserved move then provides negative evidence at
that node, although it may simply not have been tried.
Because collection is coverage-guided but non-exhaustive, we
allow exceptions through applicability slack.
The trusted missing state-action pairs are
\[
\hat{\mathcal N}
=
\left\{
(\hat{s},a)\in
\hat{\mathcal S}_{\mathrm{trust}}\times\mathcal A
\;\middle|\;
\nexists \hat{s}'\in\hat{\mathcal S}
\text{ such that }(\hat{s},a,\hat{s}')\in\hat E
\right\}.
\]

The desired zero-slack formulation is
\[
\begin{aligned}
\text{find} \quad
    & \hat X
      =\{\hat{\mathbf{x}}_{\hat{s}}\}_{\hat{s}\in\hat{\mathcal S}},
      \qquad
      \hat\Omega
      =\{\hat\omega_a\}_{a\in\mathcal A}
      \\
\text{such that} \quad
    & \hat\omega_a
      \text{ is applicable in }\hat{\mathbf{x}}_{\hat{s}_0},
      \qquad
      \forall(\hat{s}_0,a,\hat{s}_1)\in\hat E,
      \qquad \mathrm{(C1)}
      \\
    & \tau_{\hat\omega_a}(\hat{\mathbf{x}}_{\hat{s}_0})
      =\hat{\mathbf{x}}_{\hat{s}_1},
      \qquad
      \forall(\hat{s}_0,a,\hat{s}_1)\in\hat E,
      \qquad \mathrm{(C2)}
      \\
    & \hat\omega_a
      \text{ is not applicable in }\hat{\mathbf{x}}_{\hat{s}},
      \qquad
      \forall(\hat{s},a)\in\hat{\mathcal N},
      \qquad \mathrm{(C3)}.
\end{aligned}
\]
Here, $\hat X$ collects the predicate assignments,
$\hat\Omega$ collects the learned operators, and
$\tau_{\hat\omega_a}$ is the \strips transition induced by
$\hat\omega_a$.
We enforce observed applicability exactly, while relaxing
successor agreement and missing-pair exclusions with transition
and applicability slack, respectively, as described below.

\paragraph{Observed transitions provide positive evidence.}
First we encode the positive requirements $\mathrm{(C1)}, 
\mathrm{(C2)}$ as binary constraints. Each observed edge
$(\hat{s}_0,a,\hat{s}_1)\in\hat{E}$ is treated as a successful execution of
action $a$. Thus $\mathrm{(C1)}$ requires the learned preconditions $\mathbf{pre^{+}}_a, \mathbf{pre^{-}}_a$ to hold in  $\mathbf{\hat{x}_{\hat{s}_0}}$:
\[
\hat{x}_{\hat{s}_0,p}
\ge
\mathit{pre}^{+}_{a,p},
\qquad
\hat{x}_{\hat{s}_0,p}
+
\mathit{pre}^{-}_{a,p}
\le
1,
\qquad
\forall (\hat{s}_0,a,\hat{s}_1)\in\hat{E},\; p\in[m].
\]
Thus, if predicate $p$ is a positive precondition of $a$, it must be true in
every observed source state of $a$; if it is a negative precondition, it must be
false.

The learned effect masks $\mathbf{add}_a$ and $\mathbf{del}_a$
must map the source predicate vector
$\hat{\mathbf{x}}_{\hat{s}_0}$ to the successor predicate vector
$\hat{\mathbf{x}}_{\hat{s}_1}$.
 
Expanding the \strips transition
$\tau_{\hat{\omega}_a}$ to linear constraints for every
$(\hat{s}_0,a,\hat{s}_1)\in\hat{E}$ and $p$:
\[
\hat{x}_{\hat{s}_1,p}
\ge
\mathit{add}_{a,p},
\qquad
\hat{x}_{\hat{s}_1,p}
+
\mathit{del}_{a,p}
\le
1,
\]
\[
\hat{x}_{\hat{s}_1,p}
-
\hat{x}_{\hat{s}_0,p}
\le
\mathit{add}_{a,p}
+
\mathit{del}_{a,p},
\qquad
\hat{x}_{\hat{s}_0,p}
-
\hat{x}_{\hat{s}_1,p}
\le
\mathit{add}_{a,p}
+
\mathit{del}_{a,p}.
\]
The first inequality forces added predicates to be true in the successor, while
the second forces deleted predicates to be false. The last two inequalities
encode inertia: if action $a$ neither adds nor deletes predicate $p$, then the
truth value of $p$ must be unchanged from $\hat{s}_0$ to $\hat{s}_1$.

Finally, we enforce the usual consistency constraints
\[
\mathit{add}_{a,p}
+
\mathit{del}_{a,p}
\le
1,
\qquad
\mathit{pre}^{+}_{a,p}
+
\mathit{pre}^{-}_{a,p}
\le
1,
\qquad
\forall a\in\mathcal{A},\; p\in[m],
\]
so an operator cannot both add and delete the same predicate, or require the
same predicate to be both true and false.

For each observed edge $\hat{e}=(\hat{s}_0,a,\hat{s}_1)$ and
predicate $p$, transition slack $\xi_{\hat{e},p}\in\{0,1\}$
relaxes the effect and inertia constraints in $\mathrm{(C2)}$.
This allows isolated violations when the learned task graph
aliases states or contains noise. The exact softened constraints
are given in \autoref{app:exact-cpsat-encoding}.

\paragraph{Trusted missing state-action pairs provide negative evidence.}
For each $(\hat{s},a)\in\hat{\mathcal N}$, we introduce binary
applicability slack $\eta_{\hat{s},a}\in\{0,1\}$ and require
at least one learned precondition to be violated unless this
slack is used:
\[
\sum_{p=1}^{m}
\left(
\mathit{pre}^{+}_{a,p}(1-\hat{x}_{\hat{s},p})
+
\mathit{pre}^{-}_{a,p}\hat{x}_{\hat{s},p}
\right)
+
\eta_{\hat{s},a}
\ge 1,
\qquad
\forall(\hat{s},a)\in\hat{\mathcal N}.
\]
The two terms in the sum count violated positive and negative
preconditions. If no precondition is violated, the constraint
requires $\eta_{\hat{s},a}=1$, allowing the action to remain
applicable. The objective penalizes these exceptions.
We linearize the products of binary variables using auxiliary
violation variables; the exact constraints are given in
\autoref{app:exact-cpsat-encoding}.

\paragraph{Lexicographic model selection.}
Many predicate/operator assignments may satisfy the semantic constraints.
We therefore choose among feasible models using a lexicographic objective:
\begin{equation}
\min_{\mathrm{lex}}
\left(
\sum_{\hat{e}\in\hat{E}}\sum_{p=1}^{m}\xi_{\hat{e},p},
\;
\sum_{(\hat{s},a)\in\hat{\mathcal{N}}}\eta_{\hat{s},a},
\;
\sum_{a\in\mathcal{A}}\sum_{p=1}^{m}
(\mathit{add}_{a,p}+\mathit{del}_{a,p}),
\;
\sum_{a\in\mathcal{A}}\sum_{p=1}^{m}
(\mathit{pre}^{+}_{a,p}+\mathit{pre}^{-}_{a,p})
\right)
\label{eq:lex-objective-main}
\end{equation}
The priority order reflects what we want from the learned model. First,
we minimize transition slack, because observed executions should be
explained whenever possible. Second, we minimize applicability slack, so
that trusted missing state-action pairs are ruled out by learned preconditions rather
than ignored. Third, we prefer sparse effects, since actions in
manipulation domains usually change only a small number of task-relevant
facts. Finally, among equally good explanations, we prefer sparse
preconditions.

\paragraph{Guarantees.}
\autoref{app:operator-learning-guarantees} establishes properties of
the CP-SAT encoding. \autoref{prop:strips-compatible} shows that every
feasible assignment induces well-formed grounded propositional
\strips operators, regardless of slack.
\autoref{lem:graph-agreement} relates slack to agreement with the
learned graph: when \(\sum\xi=0\), every observed transition is
reproduced exactly in predicate space, and \(\sum\eta\) bounds the
number of trusted missing pairs that remain applicable.
With both slack families fixed to zero,
\autoref{thm:sound-complete} establishes an exact correspondence
between feasible solutions of the core encoding and \(m\)-predicate
grounded \strips models that reproduce the observed transitions,
exclude the trusted missing pairs, and satisfy the selected
distinctness constraints.

\subsection{Visual Predicate Classification and Test-Time Planning } 
\label{sec:predicates}
The previous stage assigns a predicate vector
$\hat{\mathbf{x}}_{\hat{s}}\in\{0,1\}^m$ to each learned graph state
$\hat{s}\in\hat{\mathcal S}$. We use these assignments as supervision
for a visual predicate classifier
$b_\phi:\mathcal O\rightarrow[0,1]^m$, which predicts the learned
predicates directly from images.

Let
$\mathcal D_{\mathrm{pred}}
=\{(o_j,\hat{s}_j)\}_{j=1}^{N_{\mathrm{pred}}}$
denote the training observations paired with their learned state
identities from the visual task graph. For each pair
$(o_j,\hat{s}_j)$, the target is
$\hat{\mathbf{x}}_{\hat{s}_j}$. The classifier is trained with
binary cross-entropy against these targets and a small confidence
regularizer, as detailed in
\autoref{app:visual-predicate-classifier}.

At test time, the input is a start image $o_{\mathrm{start}}$ and a goal image
$o_{\mathrm{goal}}$. The visual predicate classifier maps them to predicate
probabilities, which we binarize as
$\hat{\mathbf{x}}_{\mathrm{start}}=\operatorname{bin}(b_\phi(o_{\mathrm{start}}))$
and
$\hat{\mathbf{x}}_{\mathrm{goal}}=\operatorname{bin}(b_\phi(o_{\mathrm{goal}}))$.

We then use breadth-first search (BFS) in the learned predicate space,
with initial state $\hat{\mathbf{x}}_{\mathrm{start}}$, goal state
$\hat{\mathbf{x}}_{\mathrm{goal}}$, and operator set $\hat{\Omega}$.
At each predicate state $x$, BFS checks whether the preconditions
learned for each action identifier $a\in\mathcal A$ hold in $x$.
When they do, it uses the corresponding learned add/delete effects
to compute the successor $x'=\tau_{\hat{\omega}_a}(x)$.
When the goal predicate vector is reached, BFS returns a shortest
action sequence $\pi=(a_1,\ldots,a_H)$ in the learned model, with
each high-level action counted as one step.

\section{Experiments \&  Results}
\begin{wraptable}{r}{0.6\textwidth}
\vspace{-4.0em}
\centering
\footnotesize
\setlength{\tabcolsep}{2.6pt}
\caption{Domain and learned abstraction statistics.}
\label{tab:domain-stats}
\begin{tabular}{lrrr|rrrrr}
\toprule
& \multicolumn{3}{c|}{\textbf{Domain Statistics}}
& \multicolumn{5}{c}{\textbf{\method Statistics}} \\
\cmidrule(lr){2-4}\cmidrule(lr){5-9}
\textbf{Domain}
& $|\mathcal A|$
& $|\mathcal D_O|$
& $|\mathcal{S}^{\star}|$
& $|\hat{\mathcal{S}}|$
& $|\hat{\mathcal{X}}|$
& $m$
& $\sum \xi$ 
& $\sum\eta$ \\
\midrule
BlocksWorld      & 18 & 5{,}000  & 16  & 16  & 16  & 9 &0& 0  \\
DinnerTable      & 70 & 12{,}000 & 101 & 101 & 101 & 35&0& 9  \\
DinnerTable Real  & 64 & 3{,}000  & 71  & 111 & 71  & 35&0& 13 \\
\bottomrule
\end{tabular}
\vspace{-1.2em}
\end{wraptable}
\paragraph{Visual Planning Domains:} We evaluate the performance of \method in three image-based rearrangement domains. \textbf{BlocksWorld} consists  of three identical blocks and three stacking regions as in \cite{ajith2026learningdiscreteabstractionsvisual}. \textbf{DinnerTable} is a simulated tabletop rearrangement problem with five distinct objects and seven possible placements. \textbf{DinnerTable Real} uses real images from the same tabletop rearrangement setting, with additional static distractor objects such as a plate, bowl, and shelf, as shown in \autoref{fig:intro}.

For evaluation, we manually construct ground-truth task graphs
$G^\star=(\mathcal S^\star,\mathcal E^\star)$ for the portion of
each domain represented in $D_{\mathcal O}$. The domain
statistics are reported in \autoref{tab:domain-stats}.
For DinnerTable and DinnerTable Real, these graphs cover only
part of the possible task space. The reported shortest-path
horizons are measured within these reference graphs.
\autoref{app:graphs-and-plans} visualizes the learned graphs
and a representative DinnerTable plan.

\paragraph{Visual Planning Baselines.} We compare STRIPS-WM against four baselines. \newline 
\textbf{WM-Rollout:} rolls out candidate action sequences with the learned
action-conditioned FSQ dynamics model and selects the sequence whose predicted code is
closest to the goal. This follows the spirit of rollout-based visual MPC and latent
world-model planning~\cite{finn2017foresight,ichter2019robot,wu2022daydreamer}. \newline

\textbf{WM-BFS:} uses breadth-first search in the FSQ code
space with the same learned dynamics model. From each code,
it predicts a successor for every action 
$a\in\mathcal A$ and searches these predicted transitions
for a path to the goal code~\cite{bagatella2021planning}.
\newline
\textbf{LSR}~\cite{lippi_enabling_2023} learns a continuous visual latent space from
image transitions, clusters latent observations into roadmap nodes, and adds directed
edges from observed transitions. It can be viewed as learning a latent task graph
$\hat{G}=(\hat{\mathcal{S}},\hat{E})$, with planning performed by graph search over the  learned task graph. \newline
\textbf{LatPlan-AMA3}~\cite{asai2020cubespace} learns binary latent image
codes and a \strips-style neural transition model with
add/delete/no-op bit effects. We adapt AMA3 to use the provided
action identifiers, so that both methods receive the same
primitive vocabulary.
\begin{figure}
    \centering
    \includegraphics[width=1\linewidth]{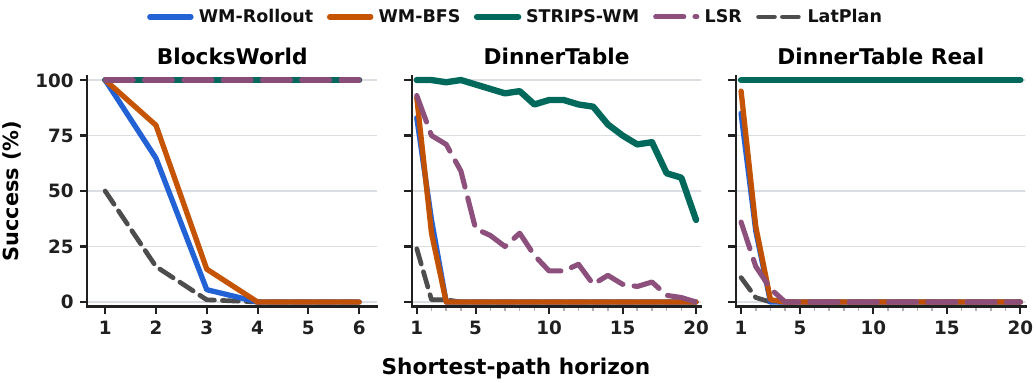}
    \vspace{-1em}
    \caption{\footnotesize Planning success as a function of shortest-path horizon. Each point is evaluated by executing the returned action sequence in the ground-truth transition graph.  Success therefore requires reaching the true goal state, not merely the goal predicted by the learned model.}
    \label{fig:experiments}
    \vspace{-2em}
\end{figure}

\paragraph{RQ1: Does \method learn the correct abstraction?}
The data in \autoref{tab:domain-stats} show how the learned task graph and predicate representation
compare to the ground-truth task graph. Let $\hat{\mathcal X}$ denote
the set of distinct predicate states produced by operator learning.
In \textbf{BlocksWorld}, the learned task graph and predicate
representation match the ground-truth state count:
$|\hat{\mathcal S}|=|\mathcal S^\star|=|\hat{\mathcal X}|=16$.
The model uses $m=9$ predicates and has zero transition and
applicability slack $(\sum\xi=0,\sum\eta=0)$. Thus, the learned
operators reproduce every observed transition in predicate space
and satisfy all applicability exclusions imposed by negative
evidence.
In \textbf{DinnerTable}, \method{} again matches the ground-truth
state count, with
$|\hat{\mathcal S}|=|\mathcal S^\star|=|\hat{\mathcal X}|=101$,
and has zero transition slack $(\sum\xi=0)$ and a small amount of
applicability slack $(\sum\eta=9)$. The observed transitions are
again represented exactly, while the applicability slack permits
exceptions to the missing-action exclusions.
In \textbf{DinnerTable Real}, the learned task graph contains
$|\hat{\mathcal S}|=111$ graph nodes
for only $|\mathcal S^\star|=71$ task states. The CP-SAT solver
maps these nodes to $|\hat{\mathcal X}|=71$ distinct predicate
states, again with zero transition slack $(\sum\xi=0)$. Operator learning therefore maps multiple graph nodes to shared predicate states while preserving agreement with all observed
transitions in predicate space.
The remaining applicability slack $(\sum\eta=13)$ permits
exceptions to the missing-action exclusions. Such exceptions can
arise when consolidated nodes have different outgoing observations:
an action observed at one node must remain applicable at their
shared predicate state, even if it was unobserved at another.

\paragraph{RQ2: How well does \method scale to long-horizon tasks?}
We evaluate image-to-plan performance using the ground-truth transition graph
\(G^{\star}=(\mathcal{S}^{\star},\mathcal{E}^{\star})\). For each domain and
shortest-path horizon \(\ell\), we sample 200 start--goal image queries whose
endpoints are distance \(\ell\) in \(G^{\star}\). Each method receives only the
start and goal images and returns either an action sequence or no plan. Returned
plans are executed in \(G^{\star}\), and success is the fraction of queries whose
plans reach the ground-truth goal state. Thus, reaching the goal only in the
learned model is not sufficient.

The results in \autoref{fig:experiments} show that \method achieves
100\% success in BlocksWorld and DinnerTable Real across the evaluated
horizons. Both DinnerTable domains have nonzero applicability slack, allowing
some actions to remain applicable at states where they were not
observed. These applications may still be valid in $G^\star$, so the
slack total alone does not determine planning success. In simulated DinnerTable, success decreases at
longer horizons despite zero transition slack. Exact agreement with
observed transitions therefore does not guarantee that every plan
generated by the learned model executes correctly in the ground-truth
graph.

The baselines degrade quickly over longer horizons. WM-Rollout and
WM-BFS use the same FSQ dynamics model, but plan through predicted code
transitions without explicit preconditions or add/delete effects.
LSR solves BlocksWorld but performs less reliably in the larger
DinnerTable domains. LatPlan-AMA3 also has lower success in these visual settings.
The comparison in \autoref{app:latplan-comparison} shows that
CP-SAT improves planning even when the task graph learned by
LatPlan is held fixed.

\Needspace{14\baselineskip}
\paragraph{RQ3: Why not plan directly on the learned task graph?}
\begin{wrapfigure}{r}{0.3\linewidth}
    \vspace{-2.0em}
    \centering
    \includegraphics[width=\linewidth]{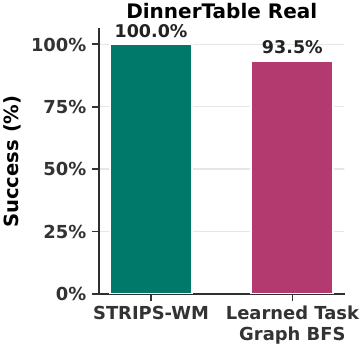}
    \vspace{-1.0em}
    \caption{\small Planning success on DinnerTable Real.}
    \label{fig:taskgraph}
    \vspace{-1.2em}
\end{wrapfigure}

We compare \method{} with breadth-first search over the learned task graph
$\hat{G}=(\hat{\mathcal S},\hat E)$ constructed in \autoref{sec:world-model}. This search follows
only the recorded edges in $\hat E$, whereas WM-BFS generates
successor codes with the learned dynamics model.
As shown in \autoref{fig:taskgraph}, the learned \strips{} model
achieves a higher planning success rate on DinnerTable Real
(100\% versus 93.5\%). The CP-SAT layer can assign multiple graph nodes to shared
predicate states and learns action preconditions and effects
for symbolic planning.

\subsection{Robot Execution in Physics Simulation}
\label{sec:robot-execution}

\begin{figure}[htbp]
    \centering
    \includegraphics[width=1\linewidth]{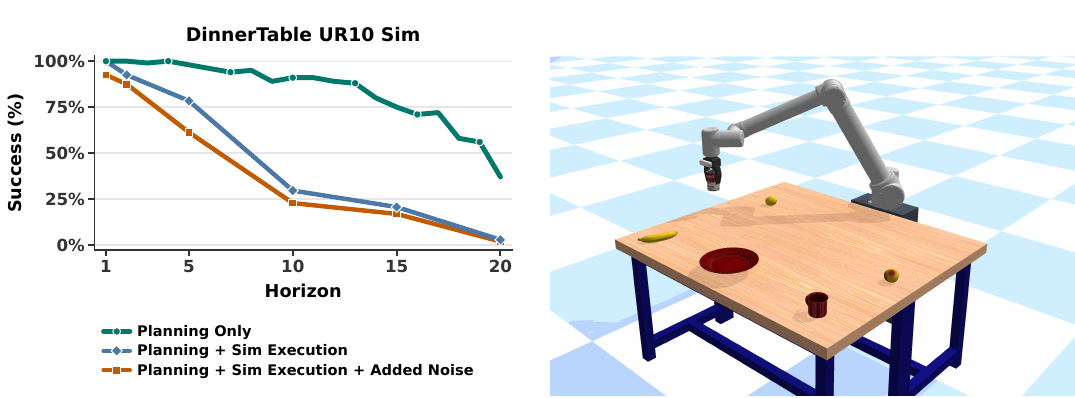}
\caption{Execution results for DinnerTable with a UR10 in PyBullet. The left panel compares planning-only success with combined planning-and-execution success, with and without added noise. The right panel shows the simulated tabletop environment used in these experiments.}
    \label{fig:robot-execution}
\end{figure}

We evaluate \method in the DinnerTable setting with a UR10 robot in PyBullet, using the \texttt{pybullet-planning}
stack~\cite{garrett2018pybulletplanning}. The robot uses an RH-P12-RN gripper that holds objects through contact and friction, allowing slips, drops, and placement errors to affect execution. After each primitive, we render a new RGB observation and use the visual predicate classifier to estimate the resulting symbolic state.

\autoref{fig:robot-execution} compares planning-only success
with combined planning-and-execution success over task horizons
from 1 to 20. We also evaluate a noisy condition that introduces up to 25\,mm of error into the planner's object-pose belief, randomizes finger friction, and shifts the camera by up to 15\,mm. 

Combined planning-and-execution success is lower than planning-only
success, particularly at longer horizons. Failed grasps and placements
can be recovered from when the resulting configuration remains within
the learned 101-state abstraction: we estimate its predicate state
from a new image and replan. On observations within this abstraction,
the visual predicate classifier achieves 100\% accuracy without added
perturbations and 99.7\% with perturbations. Execution errors can also leave
objects between placement regions, producing configurations outside
the learned abstraction. Visual grounding becomes unreliable on these
outcomes, and the system was unable to recover from them in our
experiments.

\section{Limitations}
\method{} has three main limitations. First, its learned \strips{}
operators assume deterministic transitions and do not model stochastic
execution outcomes. Second, the learned model depends on the quality
and coverage of $\hat{G}$: assigning different task states to the same graph node or
incomplete outgoing action coverage can produce incorrect preconditions or unreliable
negative evidence. Third, \method{} learns grounded propositional
operators rather than lifted relational schemas, so it does not
generalize compositionally to new object sets or action arguments
without additional data.

\acknowledgments{We thank the area chair and reviewers for their feedback,
which helped improve the manuscript. This work was supported in part
by NSF CRII Grant No.~2451108 and WPI funds.
}

\bibliography{SymbolicWMs}  

\appendix

\section*{Appendix}
\setcounter{section}{0}
\section{Architecture and Implementation Details}
\label{app:architecture}
This section gives additional implementation details for the three
learned components of \method: the image-grounded task graph in
\autoref{sec:world-model}, the operator-learning stage in
\autoref{sec:operator-learning}, and the visual predicate classifier in
\autoref{sec:predicates}. 

\subsection{Image-grounded task graph model}

To construct the image-grounded task graph
\(\hat{G}=(\hat{\mathcal{S}},\hat{E})\), we learn a discrete state
assignment
\(\alpha:\mathcal{O}\rightarrow\hat{\mathcal{S}}\). As shown in
\autoref{fig:worldmodel}, the model has a student image encoder, an EMA
teacher encoder, a finite-scalar-quantized latent code, an
action-conditioned latent dynamics model, and an inverse dynamics head.
The student encoder \(f_\theta\) maps the current observation \(o_t\) to
a continuous latent vector. Finite scalar quantization then maps this
vector to a discrete code
\[
z_t = q(f_\theta(o_t)).
\]
The successor observation \(o_{t+1}\) is encoded by an exponential-moving
average teacher encoder \(f_{\bar\theta}\). The teacher parameters are
updated after each optimization step as
\[
\bar{\theta}\leftarrow \mu \bar{\theta} + (1-\mu)\theta ,
\]
with \(\mu=0.996\). The teacher output is quantized with the same FSQ
quantizer to produce the target successor code
\[
z^{\mathrm{tar}}_{t+1}=q(f_{\bar\theta}(o_{t+1})).
\]
The action-conditioned latent dynamics model predicts a successor code
\(\hat{z}_{t+1}\) from \(z_t\) and the action identifier \(a_t\). The
inverse model predicts \(a_t\) from the current and successor codes,
which discourages collapsed representations that ignore action-relevant
visual differences.

We train this architecture with the following objective:
\[
\mathcal{L}_{\mathrm{total}}
=
\mathcal{L}_{\mathrm{dyn}}
+
\lambda_{\mathrm{inv}}\mathcal{L}_{\mathrm{inv}}
+
\lambda_{\mathrm{com}}\mathcal{L}_{\mathrm{com}}
+
\lambda_{\mathrm{aug}}\mathcal{L}_{\mathrm{aug}}
+
\lambda_{\mathrm{sep}}\mathcal{L}_{\mathrm{sep}}.
\]
The dynamics loss is
\[
\mathcal{L}_{\mathrm{dyn}}
=
\|\hat{z}_{t+1}-z^{\mathrm{tar}}_{t+1}\|_2^2,
\]
and trains the latent dynamics model to predict the teacher successor
code. The inverse loss is
\[
\mathcal{L}_{\mathrm{inv}}
=
\operatorname{CE}(\hat{a}_t,a_t),
\]
where \(\hat{a}_t\) is the inverse model prediction. This term encourages
the code to preserve information that distinguishes which action caused
a transition.

The commitment loss stabilizes the quantized representation. Let
\(\tilde{z}_t\) denote the continuous latent value immediately before
straight-through rounding, and let \(z_t\) denote the corresponding
quantized code. We use
\[
\mathcal{L}_{\mathrm{com}}
=
\|\tilde{z}_t-\operatorname{stopgrad}(z_t)\|_2^2.
\]
This pulls the encoder output toward the discrete FSQ grid point it is
assigned to, reducing oscillation near quantization boundaries while
leaving the quantized target fixed.

For real-image training, we also used mild same-image augmentation
consistency. If \(\operatorname{aug}(o_t)\) is an augmented version of the
same observation, then \(\mathcal{L}_{\mathrm{aug}}\) encourages
\(o_t\) and \(\operatorname{aug}(o_t)\) to receive the same latent code.
This term is used as a visual invariance regularizer.

\paragraph{Latent nondeterminism.}
We also monitor the determinism of the learned task graph.
Given the encoded transitions $(z_i,a_i,z'_i)$, a state--action
pair $(z,a)$ is nondeterministic in this graph if the same current
code and action identifier have more than one observed successor code:
\[
\left|\{z' \mid (z,a,z') \text{ is observed}\}\right| > 1.
\]
This diagnostic uses only the learned codes and the action identifiers
in $D_{\mathcal O}$; it does not require ground-truth symbolic states,
object metadata, or manually labeled predicates.

When two training examples have the same current code and action
identifier but different successor codes, the predictive separation
term $\mathcal L_{\mathrm{sep}}$ encourages their current latent
representations to move apart before quantization.
This encourages the representation to preserve distinctions needed
to predict the observed transitions.
\paragraph{Latent sweep and checkpoint selection.}
For each domain, we swept the FSQ latent dimensionality while keeping a
centered five-level scalar grid in each active dimension. The candidate
latent spaces were
\[
[5]^5,\qquad [5]^6,\qquad [5]^7,
\]
corresponding to \(B\in\{5,6,7\}\) FSQ dimensions with five levels per
dimension. We selected checkpoints using the total training loss and
latent-nondeterminism diagnostics, preferring lower latent nondeterminism
among comparable low-loss checkpoints.

\begin{table}[h]
\centering
\small
\begin{tabular}{lccc}
\toprule
Setting & BlocksWorld & DinnerTable & DinnerTable Real \\
\midrule
Encoder backbone & ConvNeXt-Tiny & ConvNeXt-Tiny & ConvNeXt-Tiny \\
Image resolution & \(228\times 228\) & \(228\times 228\) & \(228\times 228\) \\
FSQ latent sweep & \([5]^5,[5]^6,[5]^7\) & \([5]^5,[5]^6,[5]^7\) & \([5]^5,[5]^6,[5]^7\) \\
Selected FSQ levels & \([5]^5\) & \([5]^7\) & \([5]^7\) \\
Optimizer & AdamW & AdamW & AdamW \\
Learning rate & \(8\times 10^{-4}\) & \(8\times 10^{-4}\) & \(8\times 10^{-4}\) \\
Batch size & 64 & 64 & 64 \\
Training epochs & 170 & 170 & 170 \\
Teacher EMA decay \(\mu\) & 0.996 & 0.996 & 0.996 \\
Inverse loss weight \(\lambda_{\mathrm{inv}}\) & 5.0 & 5.0 & 5.0 \\
Commitment loss weight \(\lambda_{\mathrm{com}}\) & 0.05 & 0.05 & 0.05 \\
Predictive separation weight \(\lambda_{\mathrm{sep}}\) & 0.03 & 0.03 & 0.03 \\
\bottomrule
\end{tabular}
\caption{Image-grounded task-graph training hyperparameters. We swept the FSQ latent dimensionality
for each domain and selected the final checkpoint using training loss and latent-graph diagnostics.}
\label{tab:stage1-hyperparams}
\end{table}

\subsection{Operator-learning hyperparameters}

Given the learned task graph \(\hat{G}=(\hat{\mathcal S},\hat E)\), the
operator-learning stage assigns a binary predicate vector
\(\hat{\mathbf{x}}_{\hat{s}}\in\{0,1\}^{m}\) to each learned graph node
\(\hat{s}\in\hat{\mathcal S}\), and learns one grounded propositional
operator \(\hat{\omega}_a\) for each action identifier
\(a\in\mathcal A\). We use the CP-SAT formulation described in
\autoref{sec:operator-learning}, with the exact linearized constraints
given in \autoref{app:exact-cpsat-encoding}.

The main hyperparameter in this stage is the predicate dimension \(m\).
For each domain, we swept \(m\) and selected the smallest predicate
dimension that produced a compact STRIPS explanation within the solver
budget. In particular, we prioritized models that
reproduce the observed graph transitions with zero transition slack, and
then used the lexicographic objective in
\autoref{eq:lex-objective-main} to rank models by applicability
slack, effect sparsity, and precondition sparsity, in that order. Each fixed-\(m\) CP-SAT run was
given a total wall-clock budget of 300 seconds. This budget is applied to
the complete solver run for that candidate predicate dimension, rather
than to each individual objective term.
The slack variables have different interpretations. Transition slack is
represented by variables \(\xi_{\hat e,p}\), one for each observed edge
\(\hat e=(\hat{s},a,\hat{s}')\in\hat E\) and predicate \(p\in[m]\).
These variables relax the requirement that the learned add/delete/inertia
effects exactly reproduce the observed successor predicate value. The
aggregate transition slack reported in the experiments is
\[
\sum \xi
=
\sum_{\hat e\in\hat E}\sum_{p=1}^{m}\xi_{\hat e,p}.
\]

Thus, \(\sum \xi=0\) means that every observed edge in the learned graph
is reproduced exactly in predicate space. Applicability slack is
represented by variables \(\eta_{\hat{s},a}\), one for each trusted
missing state-action pair \((\hat{s},a)\in\hat{\mathcal N}\). The
aggregate applicability slack is
\[
\sum \eta
=
\sum_{(\hat{s},a)\in\hat{\mathcal N}}\eta_{\hat{s},a}.
\]
Thus, \(\sum\eta\) is an upper bound on the number of selected
missing state-action pairs that remain applicable.
A positive applicability slack variable permits its pair to
remain applicable, but does not require this.

\begin{table}[h]
\centering
\small
\begin{tabular}{lccc}
\toprule
Setting & BlocksWorld & DinnerTable & DinnerTable Real \\
\midrule
Solver backend & CP-SAT & CP-SAT & CP-SAT \\
Wall-clock budget per fixed-\(m\) run & 300s & 300s & 300s \\
Selected predicate dimension \(m\) & 9 & 35 & 35 \\
Transition slack in selected solution \(\sum \xi\) & 0 & 0 & 0 \\
Applicability slack in selected solution \(\sum \eta\) & 0 & 9 & 13 \\
\bottomrule
\end{tabular}
\caption{Operator-learning settings and selected CP-SAT solutions.}
\label{tab:operator-learning-hyperparams}
\end{table}

\subsection{Visual predicate classifier}
\label{app:visual-predicate-classifier}
The operator-learning stage assigns a predicate vector
$\hat{\mathbf{x}}_{\hat{s}}\in\{0,1\}^{m}$
to each learned graph node $\hat{s}\in\hat{\mathcal S}$.
We train a visual predicate classifier
$b_\phi:\mathcal O\rightarrow[0,1]^m$
to predict these vectors from images, using the learned
assignments as supervision, as described in
\autoref{sec:predicates}.

Let
\[
g_\phi(o)\in\mathbb{R}^{m}
\]
denote the classifier logits, and let
\[
b_\phi(o)=\sigma(g_\phi(o))
\]
denote the corresponding predicate probabilities.

For each visual transition \((o_i,a_i,o'_i)\in D_{\mathcal O}\), the
learned state assignment maps the images to graph nodes:
\[
\hat{s}_i=\alpha(o_i),
\qquad
\hat{s}'_i=\alpha(o'_i).
\]
We use the corresponding learned predicate assignments as binary
supervision vectors:
\[
y_i=\hat{x}_{\hat{s}_i},
\qquad
y'_i=\hat{x}_{\hat{s}'_i}.
\]
Thus, the predicate classifier is trained only from predicate assignments
learned by \method.

We train the classifier with binary cross-entropy on both the current and
successor images:
\[
\mathcal L_{\mathrm{BCE}}
=
\frac{1}{2}
\left[
\operatorname{BCE}(g_\phi(o_i),y_i)
+
\operatorname{BCE}(g_\phi(o'_i),y'_i)
\right].
\]
Here \(\operatorname{BCE}\) denotes binary cross-entropy with logits. We
also use a small confidence regularizer on the logits:
\[
\mathcal L_{\mathrm{margin}}
=
\frac{1}{2}
\left[
\operatorname{ReLU}(\gamma-|g_\phi(o_i)|)
+
\operatorname{ReLU}(\gamma-|g_\phi(o'_i)|)
\right],
\]
averaged over predicates. This term encourages predictions to move away from the binarization
threshold while the binary cross-entropy term anchors them to the learned
predicate assignments. The reported predicate-classifier objective is
\[
\mathcal L_{\mathrm{pred}}
=
\mathcal L_{\mathrm{BCE}}
+
\lambda_{\mathrm{margin}}\mathcal L_{\mathrm{margin}}.
\]

At test time, the classifier maps a start image and goal image to
predicate probabilities. We binarize each predicate independently:
\[
\hat{x}_{\mathrm{start}}
=
\operatorname{bin}(b_\phi(o_{\mathrm{start}})),
\qquad
\hat{x}_{\mathrm{goal}}
=
\operatorname{bin}(b_\phi(o_{\mathrm{goal}})),
\]
where
\[
\operatorname{bin}(b)_p =
\begin{cases}
1, & b_p \ge 0.5,\\
0, & b_p < 0.5.
\end{cases}
\]
Planning then uses \(\hat{x}_{\mathrm{start}}\),
\(\hat{x}_{\mathrm{goal}}\), and the learned operator set
\(\hat{\Omega}\) exactly as described in \autoref{sec:predicates}.

\begin{table}[h]
\centering
\small
\begin{tabular}{lccc}
\toprule
Setting & BlocksWorld & DinnerTable & DinnerTable Real \\
\midrule
Classifier backbone & ConvNeXt-Tiny &  ConvNeXt-Tiny &  ConvNeXt-Tiny \\
Input resolution & \(228\times 228\) & \(228\times 228\) & \(228\times 228\) \\
Output dimension & \(m=9\) & \(m=35\) & \(m=35\) \\
Optimizer & Adam & Adam & Adam \\
Learning rate & \(10^{-3}\) & \(10^{-3}\) & \(10^{-3}\) \\
Batch size & 64 & 64 & 64 \\
Training epochs & 50 & 150 & 150 \\
Margin \(\gamma\) & 1.0 & 1.0 & 1.0 \\
Margin weight \(\lambda_{\mathrm{margin}}\) & 0.1 & 0.1 & 0.1 \\
Binarization threshold & 0.5 & 0.5 & 0.5 \\
\bottomrule
\end{tabular}
\caption{Visual predicate classifier hyperparameters. The output dimension is the selected predicate dimension \(m\) from \autoref{tab:domain-stats}.}
\label{tab:predicate-classifier-hyperparams}
\end{table}

\section{Additional Theoretical Details}

\subsection{Exact CP-SAT Encoding for Operator Learning}
\label{app:exact-cpsat-encoding}
This section gives the binary constraints for the operator-learning formulation in \autoref{sec:operator-learning}. The input is the learned task graph \(\hat{G}=(\hat{\mathcal S},\hat E)\), together
with the trusted missing state-action pairs \(\hat{\mathcal N}\). For a
fixed predicate dimension \(m\), the CP-SAT solver assigns predicate
vectors to learned graph nodes and learns one grounded propositional
\strips operator for each action identifier.


For every learned graph node \(\hat{s}\in\hat{\mathcal S}\) and
predicate \(p\in[m]\), we introduce a binary predicate variable
\[
\hat{x}_{\hat{s},p}\in\{0,1\}.
\]
The full predicate vector for \(\hat{s}\) is
\[
\hat{\mathbf{x}}_{\hat{s}}
=
(\hat{x}_{\hat{s},1},\ldots,\hat{x}_{\hat{s},m})
\in\{0,1\}^{m}.
\]
For every action identifier \(a\in\mathcal A\) and predicate \(p\in[m]\),
we introduce binary operator-mask variables
\[
\mathit{add}_{a,p},\mathit{del}_{a,p}\in\{0,1\},
\qquad
\mathit{pre}^{+}_{a,p},\mathit{pre}^{-}_{a,p}\in\{0,1\}.
\]
We also use transition slack variables
\(\xi_{\hat e,p}\in\{0,1\}\). Here
\(\hat e=(\hat{s},a,\hat{s}')\in\hat E\) denotes an observed edge in the
learned task graph, and \(p\in[m]\) denotes a predicate. Thus
\(\xi_{\hat e,p}\) is the per-edge/per-predicate form of the aggregate
transition slack \(\sum\xi\) reported in the main text. We use
applicability slack variables
\(\eta_{\hat{s},a}\in\{0,1\}\) for trusted missing pairs
\((\hat{s},a)\in\hat{\mathcal N}\), and auxiliary violation and
distinctness variables introduced below.

\paragraph{Observed applications satisfy learned preconditions.}
Each observed edge
\[
\hat{e}=(\hat{s},a,\hat{s}')\in\hat E
\]
records a successful execution of action identifier \(a\) from
\(\hat{s}\) to \(\hat{s}'\). The action identifier \(a\) is the edge
label. Therefore, the learned preconditions of \(a\) must hold in the
predicate vector of the source node \(\hat{s}\). For every
\(\hat e=(\hat{s},a,\hat{s}')\in\hat E\) and every \(p\in[m]\), we impose
\begin{align}
\hat{x}_{\hat{s},p} &\ge \mathit{pre}^{+}_{a,p},
\label{eq:obs-pre-pos-app} \\
\hat{x}_{\hat{s},p} + \mathit{pre}^{-}_{a,p} &\le 1.
\label{eq:obs-pre-neg-app}
\end{align}
Thus, if predicate \(p\) is a positive precondition of \(a\), it must be
true in every observed source node for \(a\). If \(p\) is a negative
precondition, it must be false in every observed source node for \(a\).

\paragraph{Observed edges follow \strips effects with inertia.}
For each observed edge
\(\hat{e}=(\hat{s},a,\hat{s}')\in\hat E\) and each predicate
\(p\in[m]\), we impose
\begin{align}
\hat{x}_{\hat{s}',p} + \xi_{\hat{e},p}
&\ge \mathit{add}_{a,p},
\label{eq:obs-add-app} \\
\hat{x}_{\hat{s}',p} + \mathit{del}_{a,p}
&\le 1 + \xi_{\hat{e},p},
\label{eq:obs-del-app} \\
\hat{x}_{\hat{s}',p} - \hat{x}_{\hat{s},p}
&\le \mathit{add}_{a,p} + \mathit{del}_{a,p} + \xi_{\hat{e},p},
\label{eq:obs-inertia-1-app} \\
\hat{x}_{\hat{s},p} - \hat{x}_{\hat{s}',p}
&\le \mathit{add}_{a,p} + \mathit{del}_{a,p} + \xi_{\hat{e},p}.
\label{eq:obs-inertia-2-app}
\end{align}
The add/delete masks and positive/negative precondition masks are
disjoint:
\begin{equation}
\mathit{add}_{a,p} + \mathit{del}_{a,p} \le 1,
\qquad
\mathit{pre}^{+}_{a,p} + \mathit{pre}^{-}_{a,p} \le 1,
\label{eq:disjoint-app}
\end{equation}
for all \(a\in\mathcal A\) and \(p\in[m]\). When
\(\xi_{\hat{e},p}=0\), these constraints implement the standard bitwise
\strips update:
\[
\hat{x}_{\hat{s}',p}=
\begin{cases}
1, & \mathit{add}_{a,p}=1,\\
0, & \mathit{del}_{a,p}=1,\\
\hat{x}_{\hat{s},p}, &
\mathit{add}_{a,p}=\mathit{del}_{a,p}=0.
\end{cases}
\]
Thus, add effects set predicates to true, delete effects set predicates
to false, and all other predicates obey inertia.

\paragraph{Negative evidence from trusted missing pairs.}
The trusted missing state-action pairs are
\[
\hat{\mathcal N}
=
\{(\hat{s},a)\mid
\hat{s}\in\hat{\mathcal S}_{\mathrm{trust}},
\ \nexists \hat{s}' \text{ such that }
(\hat{s},a,\hat{s}')\in\hat E\}.
\]
For each \((\hat{s},a)\in\hat{\mathcal N}\), the learned preconditions
of action \(a\) should be violated in \(\hat{\mathbf{x}}_{\hat{s}}\),
unless applicability slack is used. We linearize these violations using
binary variables
\[
v^+_{\hat{s},a,p},v^-_{\hat{s},a,p}\in\{0,1\}.
\]
For every \((\hat{s},a)\in\hat{\mathcal N}\) and \(p\in[m]\), we impose
\begin{align}
v^{+}_{\hat{s},a,p} &\le \mathit{pre}^{+}_{a,p},
&
v^{+}_{\hat{s},a,p} &\le 1 - \hat{x}_{\hat{s},p},
&
v^{+}_{\hat{s},a,p} &\ge \mathit{pre}^{+}_{a,p} - \hat{x}_{\hat{s},p},
\label{eq:v-pos-app} \\
v^{-}_{\hat{s},a,p} &\le \mathit{pre}^{-}_{a,p},
&
v^{-}_{\hat{s},a,p} &\le \hat{x}_{\hat{s},p},
&
v^{-}_{\hat{s},a,p} &\ge
\mathit{pre}^{-}_{a,p} + \hat{x}_{\hat{s},p} - 1.
\label{eq:v-neg-app}
\end{align}
These constraints encode
\[
v^+_{\hat{s},a,p}
=
\mathit{pre}^{+}_{a,p}\wedge (1-\hat{x}_{\hat{s},p}),
\qquad
v^-_{\hat{s},a,p}
=
\mathit{pre}^{-}_{a,p}\wedge \hat{x}_{\hat{s},p}.
\]
We then require
\begin{equation}
\sum_{p=1}^{m}
\bigl(
v^{+}_{\hat{s},a,p}+v^{-}_{\hat{s},a,p}
\bigr)
+\eta_{\hat{s},a}
\ge 1,
\qquad
\forall(\hat{s},a)\in\hat{\mathcal N}.
\label{eq:negative-evidence-app}
\end{equation}
Thus, every trusted missing state-action pair must either violate at
least one learned precondition or pay applicability slack.

\paragraph{Optional distinctness constraints.}
Let $\mathcal P_{\mathrm{dist}}$ denote the set of unordered pairs
$\{\hat{s},\hat{r}\}$ of distinct graph nodes whose predicate
vectors are required to differ. For each pair in this set, the
two nodes must disagree on at least one learned predicate.
Including every pair enforces full distinctness, while setting
$\mathcal P_{\mathrm{dist}}=\varnothing$ imposes no pairwise
distinctness constraints. Nodes not explicitly required to
differ may share a predicate vector if the remaining constraints
permit it.

For each \(\{\hat{s},\hat{r}\}\in\mathcal P_{\mathrm{dist}}\) and
predicate \(p\in[m]\), we introduce a binary difference variable
\(d_{\hat{s},\hat{r},p}\in\{0,1\}\) and impose
\begin{align}
d_{\hat{s},\hat{r},p}
&\ge \hat{x}_{\hat{s},p} - \hat{x}_{\hat{r},p},
\label{eq:diff-1-app} \\
d_{\hat{s},\hat{r},p}
&\ge \hat{x}_{\hat{r},p} - \hat{x}_{\hat{s},p},
\label{eq:diff-2-app} \\
d_{\hat{s},\hat{r},p}
&\le \hat{x}_{\hat{s},p} + \hat{x}_{\hat{r},p},
\label{eq:diff-3-app} \\
d_{\hat{s},\hat{r},p}
&\le 2 - \hat{x}_{\hat{s},p} - \hat{x}_{\hat{r},p},
\label{eq:diff-4-app} \\
\sum_{p=1}^{m} d_{\hat{s},\hat{r},p}
&\ge 1.
\label{eq:state-distinct-app}
\end{align}
Together, these constraints enforce
\[
\hat{\mathbf{x}}_{\hat{s}}\neq\hat{\mathbf{x}}_{\hat{r}}
\]
for each selected pair
\(\{\hat{s},\hat{r}\}\in\mathcal P_{\mathrm{dist}}\).


\paragraph{Solver-side regularization.}
The constraints above are the semantic core of the formulation.
The implementation also uses predicate-frequency symmetry breaking
and anti-collapse constraints, while the choice of
\(\mathcal P_{\mathrm{dist}}\) determines which node distinctions
are enforced.
These choices can restrict feasible assignments or affect how
efficiently a solution is found, while preserving the \strips form
of the learned operators. The minimum-unique-state condition in
\autoref{tab:negative-evidence-ablation} bounds the total number
of distinct predicate vectors without specifying particular
graph-node pairs that must differ.
\autoref{thm:sound-complete} concerns the core constraints and selected
distinctness pairs explicitly listed in its statement.

\subsection{Formal Guarantees for Operator Learning}
\label{app:operator-learning-guarantees}
The guarantees below are stated for a fixed learned task graph
\(\hat{G}=(\hat{\mathcal S},\hat E)\) and the selected trusted missing
pairs \(\hat{\mathcal N}\). They characterize the CP-SAT encoding and
agreement with these graph constraints in the learned predicate space.
They do not establish recovery of the underlying task abstraction.

\paragraph{Trusted missing pairs.}
In our experiments, we set
\[
\hat{\mathcal S}_{\mathrm{trust}}
=
\left\{
\hat{s}\in\hat{\mathcal S}
\;\middle|\;
\exists a\in\mathcal A,\,
\exists\hat{s}'\in\hat{\mathcal S}
\text{ such that }
(\hat{s},a,\hat{s}')\in\hat E
\right\}.
\]

Thus, the trusted set contains every learned state with at least one
observed outgoing action, including all states in both DinnerTable
graphs. Collection is coverage-guided but non-exhaustive, so membership
in \(\hat{\mathcal S}_{\mathrm{trust}}\) does not establish complete
outgoing action coverage. Missing pairs from this set serve as a soft
negative prior on unobserved applications.
All statements involving \(\hat{\mathcal N}\) are relative to this
selected set.

For any feasible assignment, define
\[
\hat{\mathbf{x}}_{\hat{s}}
=
(\hat{x}_{\hat{s},1},\ldots,\hat{x}_{\hat{s},m})
\in\{0,1\}^{m}
\]
for each \(\hat{s}\in\hat{\mathcal S}\), and define
\[
\operatorname{Add}(a)
=
\{p\in[m]\mid \mathit{add}_{a,p}=1\},
\qquad
\operatorname{Del}(a)
=
\{p\in[m]\mid \mathit{del}_{a,p}=1\},
\]
\[
\operatorname{Pre}^{+}(a)
=
\{p\in[m]\mid \mathit{pre}^{+}_{a,p}=1\},
\qquad
\operatorname{Pre}^{-}(a)
=
\{p\in[m]\mid \mathit{pre}^{-}_{a,p}=1\}.
\]
The induced grounded propositional \strips operator for action identifier
\(a\) is
\[
\hat{\omega}_a
=
\bigl\langle
\operatorname{Pre}^{+}(a),
\operatorname{Pre}^{-}(a),
\operatorname{Add}(a),
\operatorname{Del}(a)
\bigr\rangle.
\]
Its successor function is
\[
\tau_{\hat{\omega}_a}(\mathbf{x})_p
=
\begin{cases}
1, & p\in\operatorname{Add}(a),\\
0, & p\in\operatorname{Del}(a),\\
x_p, & \text{otherwise.}
\end{cases}
\]
Action \(a\) is applicable in \(\mathbf{x}\) if
\(x_p=1\) for all \(p\in\operatorname{Pre}^{+}(a)\) and
\(x_p=0\) for all \(p\in\operatorname{Pre}^{-}(a)\).

\begin{proposition}[Structural \strips compatibility]
\label{prop:strips-compatible}
Any feasible assignment to
\autoref{eq:obs-pre-pos-app}--\autoref{eq:negative-evidence-app}
induces a family of well-formed grounded propositional
\strips operators
\(\{\hat{\omega}_a\}_{a\in\mathcal A}\).
In particular, for every \(a\in\mathcal A\),
\[
\operatorname{Add}(a)\cap\operatorname{Del}(a)=\varnothing,
\qquad
\operatorname{Pre}^{+}(a)\cap\operatorname{Pre}^{-}(a)
=\varnothing.
\]
This statement holds regardless of the transition and applicability
slack values.
\end{proposition}

\begin{proof}
The disjointness constraints in \autoref{eq:disjoint-app} prevent
any predicate from being both added and deleted, or required to be
both true and false. Neither transition nor applicability slack
relaxes these constraints.
\end{proof}

The two slack families control different forms of agreement with the
learned graph.

\begin{lemma}[Agreement with the learned graph]
\label{lem:graph-agreement}
For any feasible assignment, the following properties hold.
\begin{enumerate}
    \item For every observed edge
    \(\hat e=(\hat{s},a,\hat{s}')\in\hat E\),
    action \(a\) is applicable in
    \(\hat{\mathbf{x}}_{\hat{s}}\), and
    \begin{equation}
    d_H\!\left(
    \tau_{\hat{\omega}_a}(\hat{\mathbf{x}}_{\hat{s}}),
    \hat{\mathbf{x}}_{\hat{s}'}
    \right)
    \le
    \sum_{p=1}^{m}\xi_{\hat e,p},
    \label{eq:transition-agreement-bound}
    \end{equation}
    where \(d_H\) denotes Hamming distance.

    \item For every trusted missing pair
    \((\hat{s},a)\in\hat{\mathcal N}\),
    \(\eta_{\hat{s},a}=0\) implies that action \(a\) is not
    applicable in \(\hat{\mathbf{x}}_{\hat{s}}\).
    If \(n_{\mathrm{app}}\) counts the pairs in
    \(\hat{\mathcal N}\) that remain applicable, then
    \begin{equation}
    n_{\mathrm{app}}
    \le
    \sum_{(\hat{s},a)\in\hat{\mathcal N}}\eta_{\hat{s},a}.
    \label{eq:applicability-agreement-bound}
    \end{equation}
\end{enumerate}
In particular, zero transition slack reproduces every observed
transition exactly in predicate space, independently of the
applicability slack values.
\end{lemma}

\begin{proof}
For every observed edge,
\autoref{eq:obs-pre-pos-app} and
\autoref{eq:obs-pre-neg-app} ensure that each positive
precondition is true and each negative precondition is false in
the source predicate vector. Hence the observed action is applicable.

Fix an observed edge \(\hat e=(\hat{s},a,\hat{s}')\) and a
predicate \(p\) with \(\xi_{\hat e,p}=0\).
By \autoref{eq:disjoint-app}, the add and delete masks have
three possible configurations.
If \(\mathit{add}_{a,p}=1\),
\autoref{eq:obs-add-app} forces \(\hat{x}_{\hat{s}',p}=1\).
If \(\mathit{del}_{a,p}=1\),
\autoref{eq:obs-del-app} forces \(\hat{x}_{\hat{s}',p}=0\).
If both masks are zero,
\autoref{eq:obs-inertia-1-app} and
\autoref{eq:obs-inertia-2-app} force
\(\hat{x}_{\hat{s}',p}=\hat{x}_{\hat{s},p}\).
Thus the successor bit agrees with the \strips update whenever
\(\xi_{\hat e,p}=0\).
Every disagreeing bit therefore requires
\(\xi_{\hat e,p}=1\), which gives
\autoref{eq:transition-agreement-bound} after summing over
predicates.

For each trusted missing pair,
\autoref{eq:v-pos-app} and \autoref{eq:v-neg-app} enforce
\[
v^+_{\hat{s},a,p}
=
\mathit{pre}^{+}_{a,p}(1-\hat{x}_{\hat{s},p}),
\qquad
v^-_{\hat{s},a,p}
=
\mathit{pre}^{-}_{a,p}\hat{x}_{\hat{s},p}.
\]
These variables indicate positive- and negative-precondition
violations.
If \(\eta_{\hat{s},a}=0\),
\autoref{eq:negative-evidence-app} requires at least one
violation, so the action is inapplicable.
Conversely, an applicable pair has no violated precondition, and
\autoref{eq:negative-evidence-app} then requires
\(\eta_{\hat{s},a}=1\).
Counting such pairs gives
\autoref{eq:applicability-agreement-bound}.
\end{proof}

The bounds in \autoref{lem:graph-agreement} hold for any feasible
assignment. They can be strict because a positive slack variable
permits a disagreement but does not require one.
The applicability bound counts graph-node/action pairs in
\(\hat{\mathcal N}\), including separate pairs for nodes that
share a predicate vector.

\paragraph{Lexicographic model selection.}
The objective in \autoref{eq:lex-objective-main} ranks feasible
models by transition slack, applicability slack, effect sparsity,
and precondition sparsity.
The structural and agreement properties above hold for every
feasible assignment according to its slack values.

\begin{corollary}[Node-level nondeterminism and distinctness]
\label{cor:nondet}
Assume full state distinctness, so that
\(\hat{\mathbf{x}}_{\hat{s}_1}
\neq\hat{\mathbf{x}}_{\hat{s}_2}\)
is required whenever \(\hat{s}_1\neq\hat{s}_2\).
Also assume zero transition slack.
If the learned graph contains
\[
(\hat{s},a,\hat{s}_1)\in\hat E
\qquad\text{and}\qquad
(\hat{s},a,\hat{s}_2)\in\hat E
\]
with \(\hat{s}_1\neq\hat{s}_2\), then the formulation is infeasible.
\end{corollary}

\begin{proof}
By \autoref{lem:graph-agreement}, zero transition slack gives
zero Hamming distance on both observed edges. Hence
\[
\hat{\mathbf{x}}_{\hat{s}_1}
=
\tau_{\hat{\omega}_a}(\hat{\mathbf{x}}_{\hat{s}})
=
\hat{\mathbf{x}}_{\hat{s}_2}.
\]
This contradicts the full distinctness requirement.
\end{proof}

When distinctness is relaxed, nodes may share a predicate vector if
the remaining constraints permit it. Merging can then interact with
outgoing action coverage. If
\(\hat{\mathbf{x}}_{\hat{s}}=\hat{\mathbf{x}}_{\hat{r}}\),
action \(a\) is observed at \(\hat{s}\), and
\((\hat{r},a)\in\hat{\mathcal N}\), the hard observed-precondition
constraints make \(a\) applicable at both vectors.
Consequently, \autoref{eq:negative-evidence-app} requires
\(\eta_{\hat{r},a}=1\).
Differences in outgoing coverage among merged nodes can therefore
require applicability slack even when their observed transitions
have zero transition slack.

\begin{theorem}[Exactness of the zero-slack CP-SAT formulation]
\label{thm:sound-complete}
Fix the learned graph
\(\hat G=(\hat{\mathcal S},\hat E)\), the action vocabulary
\(\mathcal A\), the predicate dimension \(m\), the trusted
missing-pair set \(\hat{\mathcal N}\), and the selected
distinctness pairs \(\mathcal P_{\mathrm{dist}}\).
Consider the binary formulation consisting of
\autoref{eq:obs-pre-pos-app}--\autoref{eq:negative-evidence-app},
together with the selected distinctness constraints in
\autoref{eq:diff-1-app}--\autoref{eq:state-distinct-app},
with all transition slack variables \(\xi_{\hat e,p}\) and
applicability slack variables \(\eta_{\hat{s},a}\) fixed to zero.
This zero-slack formulation is feasible if and only if there exist
predicate vectors \(\hat{\mathbf{x}}_{\hat{s}}\in\{0,1\}^{m}\) for all
\(\hat{s}\in\hat{\mathcal S}\) and well-formed grounded
propositional \strips operators
\(\{\hat{\omega}_a\}_{a\in\mathcal A}\) such that:

\begin{enumerate}
    \item Every observed edge
    \((\hat{s},a,\hat{s}')\in\hat E\) satisfies the
    preconditions of \(\hat{\omega}_a\) at
    \(\hat{\mathbf{x}}_{\hat{s}}\).

    \item Every observed edge
    \((\hat{s},a,\hat{s}')\in\hat E\) satisfies
    \[
    \tau_{\hat{\omega}_a}(\hat{\mathbf{x}}_{\hat{s}})
    =
    \hat{\mathbf{x}}_{\hat{s}'}.
    \]

    \item Every trusted missing pair
    \((\hat{s},a)\in\hat{\mathcal N}\) is inapplicable in
    \(\hat{\mathbf{x}}_{\hat{s}}\).

    \item \(\hat{\mathbf{x}}_{\hat{s}}
    \neq\hat{\mathbf{x}}_{\hat{r}}\) for every selected pair
    \(\{\hat{s},\hat{r}\}\in\mathcal P_{\mathrm{dist}}\).
\end{enumerate}
\end{theorem}

\begin{proof}
For the forward direction, suppose the formulation is feasible
in the exact regime.
By \autoref{prop:strips-compatible}, the induced operators are
well formed.
By \autoref{lem:graph-agreement}, every observed edge satisfies
the learned preconditions and is reproduced exactly in predicate
space, and every trusted missing pair is inapplicable.
The selected distinctness constraints hold by construction.

For the reverse direction, suppose such predicate vectors and
well-formed grounded propositional \strips operators exist.
Assign \(\hat{x}_{\hat{s},p}\) according to the given predicate
vectors and set
\(\mathit{add}_{a,p}\), \(\mathit{del}_{a,p}\),
\(\mathit{pre}^{+}_{a,p}\), and \(\mathit{pre}^{-}_{a,p}\)
according to the operator sets.
Set all transition and applicability slack variables to zero.
Well-formedness gives \autoref{eq:disjoint-app}.
For each observed edge, the assumed precondition satisfaction gives
\autoref{eq:obs-pre-pos-app} and
\autoref{eq:obs-pre-neg-app}.
The assumed successor equality gives
\autoref{eq:obs-add-app}--\autoref{eq:obs-inertia-2-app}:
added bits are one, deleted bits are zero, and all other bits
retain their source values.

For each trusted missing pair, set \(v^+_{\hat{s},a,p}=1\)
exactly when a positive precondition is violated at predicate \(p\),
and set \(v^-_{\hat{s},a,p}=1\) exactly when a negative
precondition is violated. These binary values satisfy
\autoref{eq:v-pos-app} and \autoref{eq:v-neg-app}.
Since the action is inapplicable, at least one precondition is
violated, so \autoref{eq:negative-evidence-app} holds with
\(\eta_{\hat{s},a}=0\).

For each selected distinctness pair, set
\[
d_{\hat{s},\hat{r},p}
=
\left|\hat{x}_{\hat{s},p}-\hat{x}_{\hat{r},p}\right|.
\]
For binary predicate values, these assignments satisfy
\autoref{eq:diff-1-app}--\autoref{eq:diff-4-app}.
The assumed vector inequality gives
\(\sum_{p=1}^{m}d_{\hat{s},\hat{r},p}\ge 1\),
satisfying \autoref{eq:state-distinct-app}.
Hence the CP-SAT formulation is feasible.
\end{proof}

\paragraph{Interpretation.}
\autoref{lem:graph-agreement} applies even when the exact regime does
not hold. In \autoref{tab:domain-stats}, transition slack is zero in
all three domains, while applicability slack remains positive in the
two DinnerTable domains. These models therefore reproduce every
observed transition exactly in predicate space, while allowing
exceptions to the trusted missing-pair exclusions.

\autoref{thm:sound-complete} establishes correctness and completeness
of the encoding, in the sense of Bonet and
Geffner~\cite{bonet2020learningfirstordersymbolicrepresentations}:
zero-slack solutions correspond exactly to models satisfying the
supplied graph constraints. A zero-slack solution therefore certifies
an exact fit at the chosen predicate dimension and with the selected
distinctness constraints. Proven infeasibility of this exact encoding
rules out such a fit.

A sufficient condition for a zero-slack fit can be stated in terms
of a compatible model and outgoing action coverage.
Suppose the learned nodes can be assigned states of a deterministic
grounded \strips model with at most \(m\) predicates and one operator
per action identifier. Assume that this assignment preserves every
observed transition and gives different predicate vectors to every
pair in \(\mathcal P_{\mathrm{dist}}\).
If every action applicable at a trusted node's assigned state is
observed as an outgoing edge from that node, then every trusted
missing pair is inapplicable in this model.
The model therefore provides a zero-slack solution of the core
formulation by \autoref{thm:sound-complete}, padding with unused
predicates if necessary.
This coverage condition must hold separately at each trusted learned
node; observing an action at one node does not remove a missing pair
at another node that receives the same predicate vector.

An exact fit need not exist for every dataset, and the formulation
does not guarantee that slack vanishes as more data are collected.
Additional observations can resolve missing-action uncertainty, but
can also expose conflicts that cannot be represented with the chosen
predicate dimension and distinctness constraints. The guarantees
concern agreement with the supplied graph constraints; they do not
establish physical inapplicability of missing pairs or successful
execution of learned plans in the environment.

\section{Extended Experiments and Results}

\paragraph{RQ4: Is negative evidence necessary for learning useful preconditions?}
The main operator-learning formulation uses trusted missing state-action
pairs \(\hat{\mathcal N}\) as negative evidence. We ablate this design
choice on DinnerTable and DinnerTable Real to test whether observed
transitions alone are sufficient to learn useful preconditions.


We compare three conditions. The first is the standard
\method{} setting used in the main experiments, which includes
negative evidence. The second removes the negative-evidence
constraints while keeping the remaining solver setup unchanged.
The third also removes negative evidence, but requires at least
$K$ distinct predicate states:
\[
|\hat{\mathcal X}| \ge K,
\qquad
K=|\mathcal S^\star|,
\]
where
$\hat{\mathcal X}
=\{\hat{\mathbf{x}}_{\hat{s}}:\hat{s}\in\hat{\mathcal S}\}$
is the set of distinct predicate vectors assigned to graph nodes.
We set $K=101$ for DinnerTable and $K=71$ for DinnerTable Real,
using the ground-truth task-state count as additional information
in this diagnostic ablation. The constraint specifies only a
minimum number of distinct predicate states; it does not specify
which graph nodes should share a predicate vector or which actions
should be applicable. This condition tests whether a state-count
constraint alone can replace negative evidence.

\begin{figure}[h]
    \centering
    \includegraphics[width=\linewidth]{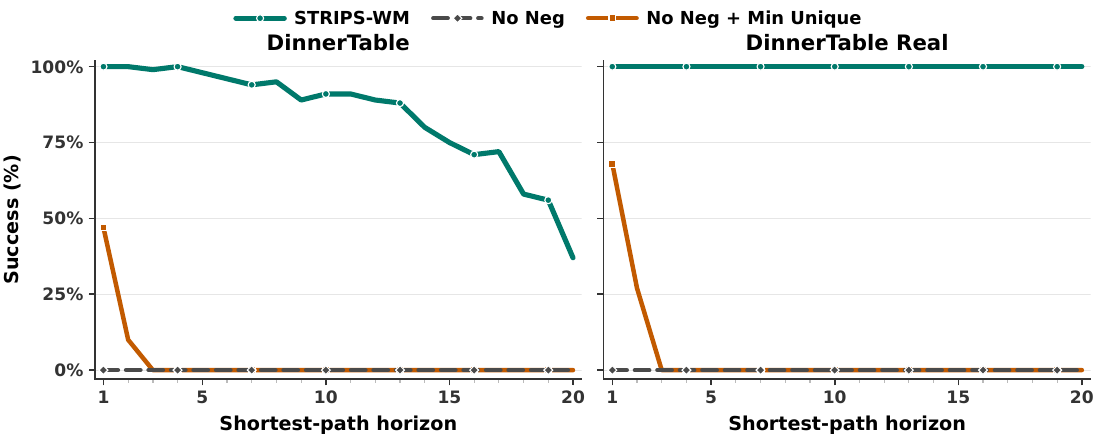}
    \vspace{-1em}
        \caption{\footnotesize Negative-evidence ablation on DinnerTable
        and DinnerTable Real. Planning performance remains poor without
        negative evidence, even when the ground-truth task-state count
        is supplied as a lower bound on the number of distinct predicate
        states.}
        \label{fig:negative-evidence-ablation}
    \vspace{-1em}
\end{figure}

\begin{table}[h]
\centering
\small
\resizebox{\linewidth}{!}{
\begin{tabular}{llrrrrr}
\toprule
Domain & Condition
& \(|\mathcal S^\star|\)
& \(\sum \xi\)
& \(\sum \eta\)
& \shortstack{Applicable\\missing pairs}
& \(|\hat{\mathcal X}|\) \\
\midrule
DinnerTable
& STRIPS-WM
& 101 & 0 & 9 & 9 & 101 \\
DinnerTable
& No neg.
& 101 & 0 & 0 & 6864 & 2 \\
DinnerTable
& No neg. + min-unique
& 101 & 0 & 0 & 6864 & 101 \\
\midrule
DinnerTable Real
& STRIPS-WM
& 71 & 0 & 13 & 13 & 71 \\
DinnerTable Real
& No neg.
& 71 & 0 & 0 & 6893 & 2 \\
DinnerTable Real
& No neg. + min-unique
& 71 & 32 & 0 & 6893 & 83 \\
\bottomrule
\end{tabular}
}
\vspace{0.3em}
\caption{Negative-evidence ablation for operator learning.
Applicable missing pairs are graph-node/action pairs that are
unobserved in the learned graph but applicable under the learned
model. For each domain, all conditions are evaluated on the
trusted missing-pair set of the standard run, counting separate
pairs for nodes that share a predicate vector.
The min-unique condition requires
$|\hat{\mathcal X}|\ge|\mathcal S^\star|$, using the ground-truth
task-state count as a lower bound in this diagnostic ablation.}
\label{tab:negative-evidence-ablation}
\end{table}

\autoref{fig:negative-evidence-ablation} shows that removing
negative evidence severely degrades planning in both domains.
Without negative evidence, both domains collapse to two distinct
predicate vectors, and planning success drops to $0\%$.
Adding the minimum-unique-state constraint prevents this
two-state collapse, but planning performance remains poor.

The operator-learning statistics in
\autoref{tab:negative-evidence-ablation} explain why preserving
enough states is not sufficient. With negative evidence, 9
missing state-action pairs in DinnerTable and 13 in DinnerTable
Real remain applicable under the learned model. Removing negative
evidence increases these counts to 6{,}864 and 6{,}893,
respectively. Both runs in the \emph{No neg.} condition have zero
transition slack, showing that the solver can fit every observed
transition while leaving many missing state-action pairs
applicable.

Adding the minimum-unique-state constraint increases the number
of distinct predicate states to 101 in DinnerTable and 83 in
DinnerTable Real, but leaves the applicable-missing-pair counts
unchanged. The DinnerTable Real model also has nonzero transition
slack $(\sum\xi=32)$ in this condition. A lower bound on the number
of predicate states therefore does not resolve the ambiguity
about action applicability. Negative evidence addresses this
ambiguity by encouraging the learned preconditions to rule out
actions that have not been observed at a state.

\subsection{Effect of Visual Discretization}
\label{app:fsq-ablation}

\begin{figure}[htbp]
    \centering
    \includegraphics[width=0.6\linewidth]{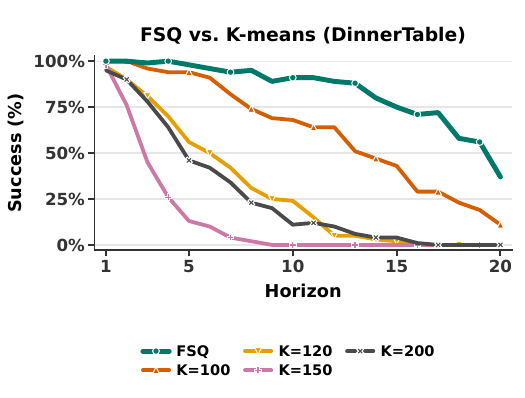}
 \caption{FSQ is compared with K-means applied to the same continuous visual
    features, using the same CP-SAT budget and evaluation queries.}    
\label{fig:fsq-ablation}
\end{figure}

The operator-learning stage takes an action-labeled discrete graph
as input and does not require a particular discretization method.
To evaluate the choice of FSQ, we replace it with K-means applied
to the same continuous visual features on DinnerTable. We consider
\(K\in\{60,80,100,120,150,200\}\), while retaining the CP-SAT
stage, its 300\,s budget, and the evaluation queries used for the
comparison.

\autoref{fig:fsq-ablation} reports planning success across
task horizons. For \(K=60\) and \(K=80\), the solver does not
return a feasible model within the allotted budget. Among the
evaluated K-means settings, \(K=100\) gives the strongest planning
performance, but remains below FSQ at longer horizons.

The comparison shows that discretization affects the graph
supplied to operator learning and, consequently, planning
performance. CP-SAT can fit predicates and operators to different
discrete abstractions, but does not remove the dependence on
which distinctions the visual representation preserves.

\subsection{Comparison with LatPlan}
\label{app:latplan-comparison}

Our setting provides a fixed action vocabulary and requires one
grounded operator per action identifier. LatPlan instead learns
its own action symbols and may assign several symbols to
transitions associated with one primitive in our dataset.
The LatPlan-AMA3 baseline therefore uses the provided action
identifiers. As an additional check, we trained native LatPlan
with action discovery on BlocksWorld. These runs learned only
two to four distinct states in this setting.

The strongest checkpoint using the fixed action vocabulary
represented the 16 BlocksWorld states with 13 learned states.
To assess the contribution of CP-SAT with the learned task graph
held fixed, we passed this unchanged 13-state graph to the
operator-learning stage. Both conditions are evaluated on the same 600 BlocksWorld planning
queries for this ablation. Successful plans increase from 67 to 113,
with every query solved by LatPlan also solved after adding CP-SAT. This comparison shows that
CP-SAT can improve planning from a fixed imperfect abstraction,
although it cannot recover distinctions between observations
that the visual representation has already mapped to the same
graph node.

\begin{figure}[htbp]
    \centering
    \includegraphics[width=0.6\linewidth]
    {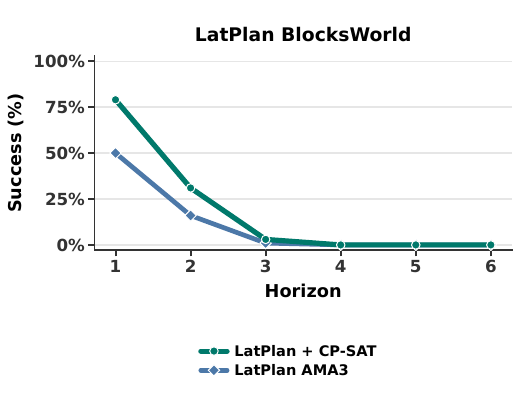}
    \caption{Planning success on BlocksWorld with the task graph learned
by LatPlan-AMA3 held fixed. Applying CP-SAT operator learning to
    the same 13-state graph increases successful plans from 67
    to 113 out of 600 queries. Every query solved by LatPlan-AMA3
    remains solved after adding CP-SAT.}
    \label{fig:latplan-cpsat}
\end{figure}

\FloatBarrier

\section{Learned Graphs and Plans}
\label{app:graphs-and-plans}

\subsection{Learned task graphs}
\label{app:learned-task-graphs}

\autoref{fig:learned-task-graphs} shows the task graphs learned in
\autoref{sec:world-model}, together with the ground-truth reference
graph for DinnerTable Real. In the learned graphs, observations
assigned to the same discrete code share a node, and observed
action-indexed transitions define directed edges. Action labels
are omitted from the visualization for readability.

The learned task graph for BlocksWorld contains 16 nodes, matching
the ground-truth state count. For DinnerTable Real, the learned
task graph contains 111 nodes for the 71 states in the reference
graph. The subsequent operator-learning stage with CP-SAT assigns these 111
nodes to 71 distinct predicate states. The learned graphs in the
figure precede operator learning, while
\autoref{tab:domain-stats} also reports the resulting predicate
representations.

\begin{figure}[tbp]
    \centering
    \includegraphics[width=\linewidth]
    {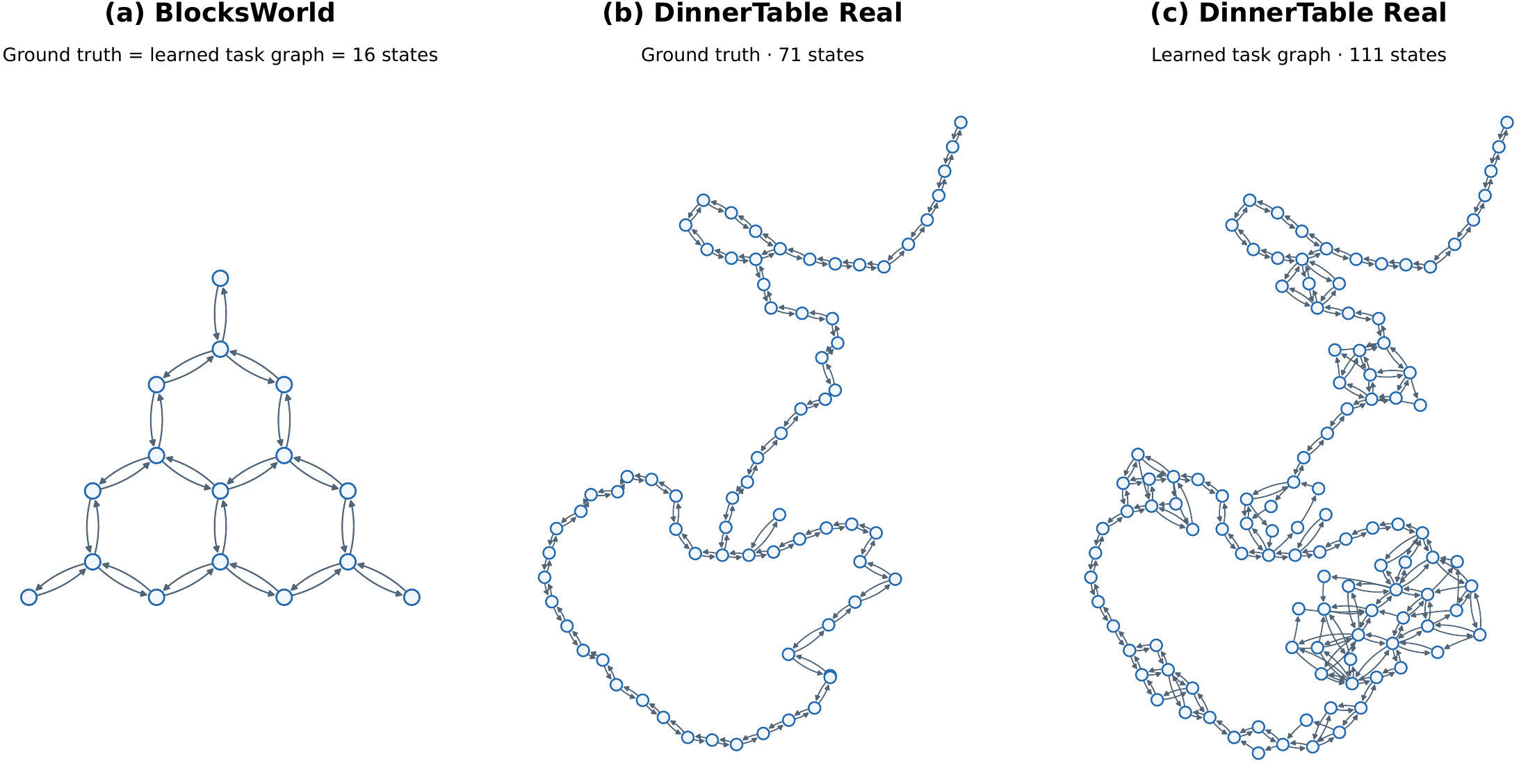}
    \caption{Task-graph visualizations.
    The panels show the learned BlocksWorld graph with 16 states,
    the ground-truth reference graph for DinnerTable Real with
    71 states, and the learned DinnerTable Real graph with
    111 nodes, respectively. Arrows indicate transition
    directions; action labels are omitted for readability.
    The learned graphs shown here precede CP-SAT operator learning.}
    \label{fig:learned-task-graphs}
\end{figure}

\FloatBarrier
\subsection{Example plan and data coverage}
\label{app:example-plan}
\begin{figure}[tbp]
    \centering
    \includegraphics[width=\linewidth]
    {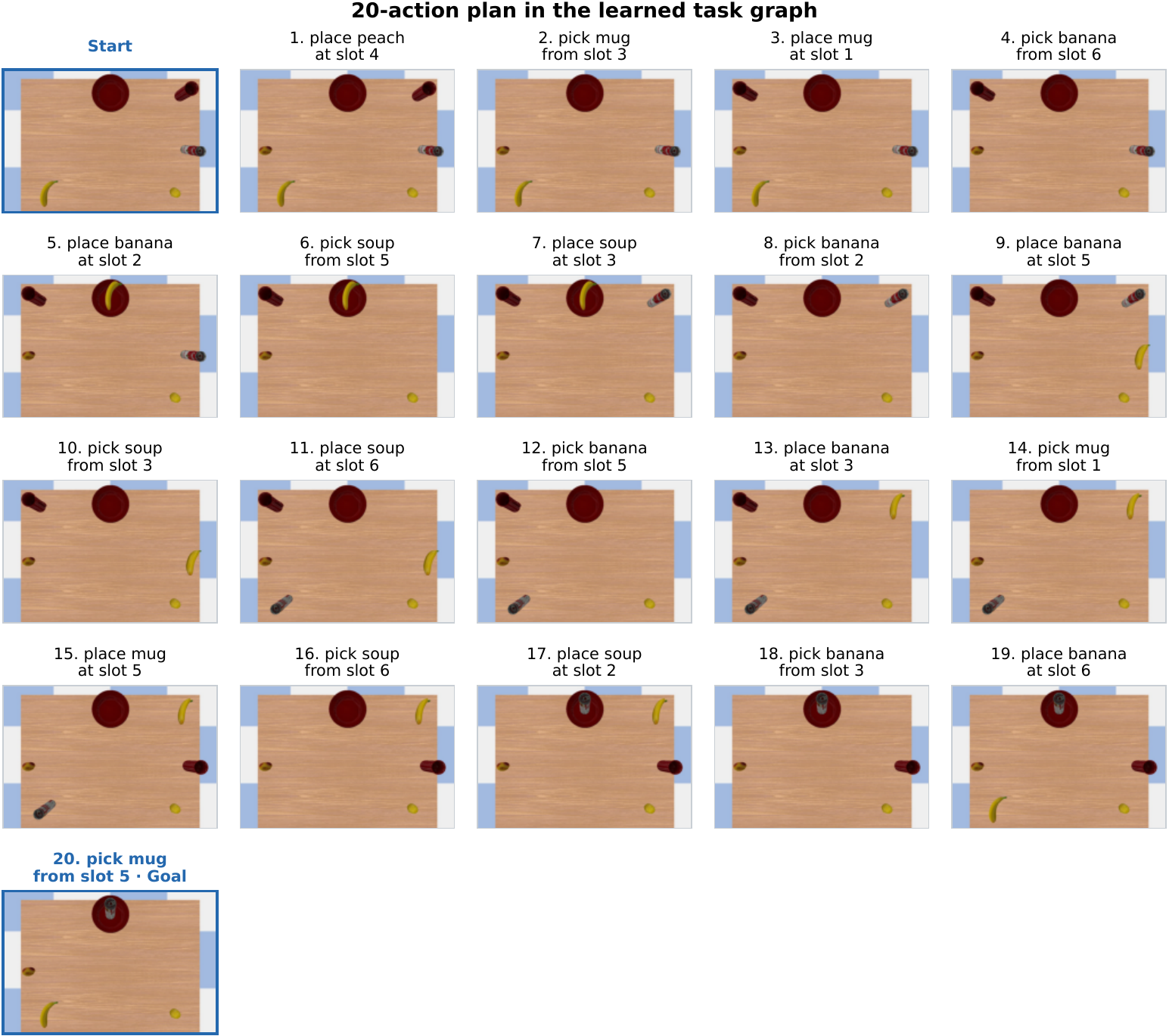}
    \caption{A 20-action DinnerTable plan through the learned
    task graph. Frames are ordered from left to right and
    top to bottom. The first frame shows the start state;
    each subsequent frame shows the state after the labeled
    action. Held objects are not shown in the tabletop views:
    the peach is held at the start, and the mug is held at the
    goal. Placement slots are numbered left to right within
    each row: 1--3 at the top, 4--5 in the middle, and 6--7 at
    the bottom. Action descriptions are provided for readability.}
    \label{fig:dinnertable-example-plan}
\end{figure}

\autoref{fig:dinnertable-example-plan} shows a 20-action plan
in DinnerTable simulation. Initially, the peach is held, the
mug is at slot 3, and the soup can is at slot 5. The goal
requires the peach at slot 4 and the soup can at slot 2,
with the mug held. The banana and lemon occupy slots 6 and 7,
respectively, in both the start and goal states. Along the
displayed route, the soup can moves through slots 3 and 6,
while the mug and banana are temporarily relocated.

For this start and goal pair, the shortest path contains 20 actions
in both the learned task graph and the ground-truth reference
graph constructed for the collected data. A direct four-action
route in the underlying rearrangement task would place the
peach at slot 4, pick the soup can from slot 5, place it at
slot 2, and pick the mug from slot 3. This route passes through
a configuration in which the soup can is held, the peach is
at slot 4, and the mug remains at slot 3, with the banana
and lemon unchanged.

That intermediate configuration was not observed during data
collection, so the transitions into and out of it are absent
from the recorded graph. The individual soup-picking and
soup-placing actions were observed in other configurations,
including those at steps 6 and 17 of the displayed plan.
The missing information is therefore the application of
these actions in the configuration needed for the shortcut.
This example illustrates how incomplete state coverage can
increase shortest-path distances even when the start and goal
states have both been observed.

\clearpage
\section{Learned PDDL domains}
\label{app:learned-pddl}
Below are the propositional PDDL domain definitions directly synthesized by STRIPS-WM 
for BlocksWorld ($m = 9$, $|\mathcal{A}| = 18$), DinnerTable simulation ($m = 35$, $|\mathcal{A}| = 70$), 
and DinnerTable Real ($m = 35$, $|\mathcal{A}| = 64$).

For the planning evaluation, we use manually constructed
ground-truth task graphs $G^\star$ covering the portions of the
domains represented in the data. Returned action sequences are
executed in these graphs to check whether they reach the ground-truth
goal state. We do not use hand-specified PDDL domains for this
evaluation; the listings below are the models learned by \method{}.

\paragraph{Notation.}
Each 0-ary predicate \texttt{(b\textit{i})} corresponds to bit $i$ of the learned state vector 
$\hat{x}_{\hat{s}} \in \{0,1\}^m$ (zero-indexed). Predicate indices are learned independently 
per domain and carry no shared semantics across tasks. Because actions are grounded primitives, 
operator parameter lists are empty. Descriptive action names are used for readability; the learning pipeline receives only the corresponding opaque identifiers. (For example: Action 0, Action 2, etc)
\Needspace{18\baselineskip}
\subsection{BlocksWorld3}
\label{app:pddl-blocksworld3}
\noindent 9 predicates; 18 grounded actions.\par

\begin{pddlblock}
\begin{lstlisting}[style=paperPDDL]
(define (domain blocksworld3_learned)
  (:requirements :strips :negative-preconditions)
  (:predicates (b0) (b1) (b2) (b3) (b4) (b5) (b6) (b7) (b8))
\end{lstlisting}
\end{pddlblock}

\begin{pddlblock}
\begin{lstlisting}[style=paperPDDL]
  (:action pick1
    :parameters ()
    :precondition (and (not (b0)) (not (b1)))
    :effect (and (b0) (b1)))
\end{lstlisting}
\end{pddlblock}

\begin{pddlblock}
\begin{lstlisting}[style=paperPDDL]
  (:action pick2
    :parameters ()
    :precondition (and (b2) (not (b0)))
    :effect (and (b0) (not (b2))))
\end{lstlisting}
\end{pddlblock}

\begin{pddlblock}
\begin{lstlisting}[style=paperPDDL]
  (:action pick3
    :parameters ()
    :precondition (and (b2) (not (b1)))
    :effect (and (b1) (not (b2))))
\end{lstlisting}
\end{pddlblock}

\begin{pddlblock}
\begin{lstlisting}[style=paperPDDL]
  (:action pick4
    :parameters ()
    :precondition (and (b2) (b4))
    :effect (and (not (b2)) (not (b4))))
\end{lstlisting}
\end{pddlblock}

\begin{pddlblock}
\begin{lstlisting}[style=paperPDDL]
  (:action pick5
    :parameters ()
    :precondition (and (b3) (not (b1)))
    :effect (and (b1) (not (b3))))
\end{lstlisting}
\end{pddlblock}

\begin{pddlblock}
\begin{lstlisting}[style=paperPDDL]
  (:action pick6
    :parameters ()
    :precondition (and (b5) (not (b0)))
    :effect (and (b0) (not (b5))))
\end{lstlisting}
\end{pddlblock}

\begin{pddlblock}
\begin{lstlisting}[style=paperPDDL]
  (:action pick7
    :parameters ()
    :precondition (and (b8))
    :effect (and (not (b8))))
\end{lstlisting}
\end{pddlblock}

\begin{pddlblock}
\begin{lstlisting}[style=paperPDDL]
  (:action pick8
    :parameters ()
    :precondition (and (b7))
    :effect (and (not (b7))))
\end{lstlisting}
\end{pddlblock}

\begin{pddlblock}
\begin{lstlisting}[style=paperPDDL]
  (:action pick9
    :parameters ()
    :precondition (and (b6))
    :effect (and (not (b6))))
\end{lstlisting}
\end{pddlblock}

\begin{pddlblock}
\begin{lstlisting}[style=paperPDDL]
  (:action place1
    :parameters ()
    :precondition (and (b0) (b1) (not (b4)) (not (b6)) (not (b7)))
    :effect (and (not (b0)) (not (b1))))
\end{lstlisting}
\end{pddlblock}

\begin{pddlblock}
\begin{lstlisting}[style=paperPDDL]
  (:action place2
    :parameters ()
    :precondition (and (b0) (not (b2)) (not (b3)) (not (b6)) (not (b8)))
    :effect (and (b2) (not (b0))))
\end{lstlisting}
\end{pddlblock}

\begin{pddlblock}
\begin{lstlisting}[style=paperPDDL]
  (:action place3
    :parameters ()
    :precondition (and (b1) (not (b2)) (not (b5)) (not (b7)) (not (b8)))
    :effect (and (b2) (not (b1))))
\end{lstlisting}
\end{pddlblock}

\begin{pddlblock}
\begin{lstlisting}[style=paperPDDL]
  (:action place4
    :parameters ()
    :precondition (and (not (b2)) (not (b3)) (not (b4)) (not (b5)))
    :effect (and (b2) (b4)))
\end{lstlisting}
\end{pddlblock}

\begin{pddlblock}
\begin{lstlisting}[style=paperPDDL]
  (:action place5
    :parameters ()
    :precondition (and (b1) (not (b3)) (not (b4)) (not (b5)))
    :effect (and (b3) (not (b1))))
\end{lstlisting}
\end{pddlblock}

\begin{pddlblock}
\begin{lstlisting}[style=paperPDDL]
  (:action place6
    :parameters ()
    :precondition (and (b0) (not (b3)) (not (b4)) (not (b5)))
    :effect (and (b5) (not (b0))))
\end{lstlisting}
\end{pddlblock}

\begin{pddlblock}
\begin{lstlisting}[style=paperPDDL]
  (:action place7
    :parameters ()
    :precondition (and (b4) (not (b2)) (not (b8)))
    :effect (and (b8)))
\end{lstlisting}
\end{pddlblock}

\begin{pddlblock}
\begin{lstlisting}[style=paperPDDL]
  (:action place8
    :parameters ()
    :precondition (and (b1) (b3) (not (b7)))
    :effect (and (b7)))
\end{lstlisting}
\end{pddlblock}

\begin{pddlblock}
\begin{lstlisting}[style=paperPDDL]
  (:action place9
    :parameters ()
    :precondition (and (b0) (b5) (not (b6)))
    :effect (and (b6)))
)
\end{lstlisting}
\end{pddlblock}

\Needspace{18\baselineskip}
\subsection{DinnerTable (simulation)}
\label{app:pddl-dinnertable-sim}
\noindent 35 predicates; 70 grounded actions.\par

\begin{pddlblock}
\begin{lstlisting}[style=paperPDDL]
(define (domain dinnertable_sim_learned)
  (:requirements :strips :negative-preconditions)
  (:predicates (b0) (b1) (b2) (b3) (b4) (b5) (b6) (b7) (b8) (b9) (b10)
      (b11) (b12) (b13) (b14) (b15) (b16) (b17) (b18) (b19) (b20) (b21)
      (b22) (b23) (b24) (b25) (b26) (b27) (b28) (b29) (b30) (b31) (b32)
      (b33) (b34))
\end{lstlisting}
\end{pddlblock}

\begin{pddlblock}
\begin{lstlisting}[style=paperPDDL]
  (:action pick_banana_from_1
    :parameters ()
    :precondition (and (b3) (b29) (b31) (not (b12)) (not (b30)))
    :effect (and (b18) (b19) (not (b3)) (not (b4)) (not (b5))))
\end{lstlisting}
\end{pddlblock}

\begin{pddlblock}
\begin{lstlisting}[style=paperPDDL]
  (:action pick_banana_from_2
    :parameters ()
    :precondition (and (b33) (not (b7)) (not (b12)) (not (b26)))
    :effect (and (b15) (b18) (b31) (not (b3)) (not (b13)) (not (b22))
      (not (b33))))
\end{lstlisting}
\end{pddlblock}

\begin{pddlblock}
\begin{lstlisting}[style=paperPDDL]
  (:action pick_banana_from_3
    :parameters ()
    :precondition (and (not (b4)) (not (b14)) (not (b22)) (not (b24))
      (not (b31)))
    :effect (and (b9) (b15) (b18) (b31)))
\end{lstlisting}
\end{pddlblock}

\begin{pddlblock}
\begin{lstlisting}[style=paperPDDL]
  (:action pick_banana_from_4
    :parameters ()
    :precondition (and (b6) (b12) (b32))
    :effect (and (b21) (b28) (not (b12)) (not (b20)) (not (b22))))
\end{lstlisting}
\end{pddlblock}

\begin{pddlblock}
\begin{lstlisting}[style=paperPDDL]
  (:action pick_banana_from_5
    :parameters ()
    :precondition (and (b22) (b28) (not (b5)) (not (b13)))
    :effect (and (b9) (b15) (b31) (not (b7)) (not (b20)) (not (b22))
      (not (b25))))
\end{lstlisting}
\end{pddlblock}

\begin{pddlblock}
\begin{lstlisting}[style=paperPDDL]
  (:action pick_banana_from_6
    :parameters ()
    :precondition (and (b12) (not (b2)) (not (b10)) (not (b24)))
    :effect (and (b31) (not (b3)) (not (b4)) (not (b11)) (not (b12))
      (not (b30))))
\end{lstlisting}
\end{pddlblock}

\begin{pddlblock}
\begin{lstlisting}[style=paperPDDL]
  (:action pick_banana_from_7
    :parameters ()
    :precondition (and (b17) (b20) (b33))
    :effect (and (b19) (b31) (not (b4)) (not (b5)) (not (b7)) (not (b8))
      (not (b20)) (not (b22)) (not (b24)) (not (b33))))
\end{lstlisting}
\end{pddlblock}

\begin{pddlblock}
\begin{lstlisting}[style=paperPDDL]
  (:action pick_lemon_from_1
    :parameters ()
    :precondition (and (b18) (b28) (b33))
    :effect (and (b12) (not (b6)) (not (b17)) (not (b21)) (not (b25))
      (not (b28))))
\end{lstlisting}
\end{pddlblock}

\begin{pddlblock}
\begin{lstlisting}[style=paperPDDL]
  (:action pick_lemon_from_2
    :parameters ()
    :precondition (and (b27) (not (b10)) (not (b24)) (not (b29)))
    :effect (and (b10) (b15)))
\end{lstlisting}
\end{pddlblock}

\begin{pddlblock}
\begin{lstlisting}[style=paperPDDL]
  (:action pick_lemon_from_3
    :parameters ()
    :precondition (and (b14) (not (b11)) (not (b17)) (not (b24)))
    :effect (and (b2) (b3) (b11) (b12) (b22) (b33) (not (b0)) (not (b1))
      (not (b6))))
\end{lstlisting}
\end{pddlblock}

\begin{pddlblock}
\begin{lstlisting}[style=paperPDDL]
  (:action pick_lemon_from_4
    :parameters ()
    :precondition (and (b2) (b4) (b6) (not (b12)) (not (b21)))
    :effect (and (b12) (b14) (b15) (not (b16)) (not (b17))))
\end{lstlisting}
\end{pddlblock}

\begin{pddlblock}
\begin{lstlisting}[style=paperPDDL]
  (:action pick_lemon_from_5
    :parameters ()
    :precondition (and (b27) (not (b7)) (not (b23)) (not (b24)))
    :effect (and (b7) (b14) (b15) (not (b20)) (not (b25))))
\end{lstlisting}
\end{pddlblock}

\begin{pddlblock}
\begin{lstlisting}[style=paperPDDL]
  (:action pick_lemon_from_6
    :parameters ()
    :precondition (and (b12) (b23) (not (b3)) (not (b24)))
    :effect (and (b5) (b7) (b15) (b24)))
\end{lstlisting}
\end{pddlblock}

\begin{pddlblock}
\begin{lstlisting}[style=paperPDDL]
  (:action pick_lemon_from_7
    :parameters ()
    :precondition (and (b3) (not (b12)) (not (b16)) (not (b21)))
    :effect (and (b2) (b11) (b12) (b14) (b15)))
\end{lstlisting}
\end{pddlblock}

\begin{pddlblock}
\begin{lstlisting}[style=paperPDDL]
  (:action pick_mug_from_1
    :parameters ()
    :precondition (and (not (b7)) (not (b8)) (not (b19)) (not (b22))
      (not (b31)))
    :effect (and (b5) (b7) (b19) (b24)))
\end{lstlisting}
\end{pddlblock}

\begin{pddlblock}
\begin{lstlisting}[style=paperPDDL]
  (:action pick_mug_from_2
    :parameters ()
    :precondition (and (b33) (not (b5)) (not (b22)))
    :effect (and (b5) (b24) (b25) (not (b6)) (not (b14)) (not (b28))
      (not (b33))))
\end{lstlisting}
\end{pddlblock}

\begin{pddlblock}
\begin{lstlisting}[style=paperPDDL]
  (:action pick_mug_from_3
    :parameters ()
    :precondition (and (b3) (b19) (b23) (not (b0)) (not (b5)))
    :effect (and (b5) (b7) (b9) (b15) (b24) (not (b10)) (not (b23))))
\end{lstlisting}
\end{pddlblock}

\begin{pddlblock}
\begin{lstlisting}[style=paperPDDL]
  (:action pick_mug_from_4
    :parameters ()
    :precondition (and (b3) (b5) (b6) (b22) (b27))
    :effect (and (b19) (b24) (b26) (not (b20)) (not (b22)) (not (b25))))
\end{lstlisting}
\end{pddlblock}

\begin{pddlblock}
\begin{lstlisting}[style=paperPDDL]
  (:action pick_mug_from_5
    :parameters ()
    :precondition (and (b3) (b5) (b20) (not (b24)) (not (b27)))
    :effect (and (b19) (b24) (b26) (not (b17)) (not (b20)) (not (b25))))
\end{lstlisting}
\end{pddlblock}

\begin{pddlblock}
\begin{lstlisting}[style=paperPDDL]
  (:action pick_mug_from_6
    :parameters ()
    :precondition (and (b17) (not (b5)) (not (b19)) (not (b34)))
    :effect (and (b5) (b19) (b24) (not (b17)) (not (b22))))
\end{lstlisting}
\end{pddlblock}

\begin{pddlblock}
\begin{lstlisting}[style=paperPDDL]
  (:action pick_mug_from_7
    :parameters ()
    :precondition (and (b5) (b14) (not (b12)) (not (b21)) (not (b24)))
    :effect (and (b2) (b3) (b19) (b21) (b24) (b28) (not (b4)) (not (b18))))
\end{lstlisting}
\end{pddlblock}

\begin{pddlblock}
\begin{lstlisting}[style=paperPDDL]
  (:action pick_peach_from_1
    :parameters ()
    :precondition (and (b11) (b16) (not (b0)) (not (b5)))
    :effect (and (b0) (b5) (b24) (b34) (not (b7))))
\end{lstlisting}
\end{pddlblock}

\begin{pddlblock}
\begin{lstlisting}[style=paperPDDL]
  (:action pick_peach_from_2
    :parameters ()
    :precondition (and (b1) (b12) (not (b5)) (not (b7)) (not (b24)))
    :effect (and (b5) (b10) (b16) (b24) (not (b6)) (not (b13)) (not (b27))
      (not (b29))))
\end{lstlisting}
\end{pddlblock}

\begin{pddlblock}
\begin{lstlisting}[style=paperPDDL]
  (:action pick_peach_from_3
    :parameters ()
    :precondition (and (b20) (b32) (not (b3)))
    :effect (and (b3) (b5) (not (b4)) (not (b6))))
\end{lstlisting}
\end{pddlblock}

\begin{pddlblock}
\begin{lstlisting}[style=paperPDDL]
  (:action pick_peach_from_4
    :parameters ()
    :precondition (and (not (b5)) (not (b6)) (not (b9)) (not (b31)))
    :effect (and (b0) (b1) (b5) (b16) (b24)))
\end{lstlisting}
\end{pddlblock}

\begin{pddlblock}
\begin{lstlisting}[style=paperPDDL]
  (:action pick_peach_from_5
    :parameters ()
    :precondition (and (b2) (b14) (b25) (not (b21)))
    :effect (and (b16) (b24) (not (b2)) (not (b6)) (not (b12)) (not (b14))
      (not (b17)) (not (b25))))
\end{lstlisting}
\end{pddlblock}

\begin{pddlblock}
\begin{lstlisting}[style=paperPDDL]
  (:action pick_peach_from_6
    :parameters ()
    :precondition (and (b5) (b7) (b34) (not (b15)))
    :effect (and (b16) (b24) (not (b7)) (not (b17)) (not (b27))))
\end{lstlisting}
\end{pddlblock}

\begin{pddlblock}
\begin{lstlisting}[style=paperPDDL]
  (:action pick_peach_from_7
    :parameters ()
    :precondition (and (b12) (b14) (not (b8)) (not (b15)))
    :effect (and (b1) (b5) (b8) (b24) (b34) (not (b7)) (not (b27))))
\end{lstlisting}
\end{pddlblock}

\begin{pddlblock}
\begin{lstlisting}[style=paperPDDL]
  (:action pick_tomato_soup_can_from_1
    :parameters ()
    :precondition (and (b21) (not (b19)) (not (b23)))
    :effect (and (b16) (b19) (b24) (b31) (not (b3)) (not (b4))))
\end{lstlisting}
\end{pddlblock}

\begin{pddlblock}
\begin{lstlisting}[style=paperPDDL]
  (:action pick_tomato_soup_can_from_2
    :parameters ()
    :precondition (and (b25) (b29) (not (b9)) (not (b13)))
    :effect (and (b10) (b16) (b19) (b21) (b24) (b31) (not (b28))
      (not (b29))))
\end{lstlisting}
\end{pddlblock}

\begin{pddlblock}
\begin{lstlisting}[style=paperPDDL]
  (:action pick_tomato_soup_can_from_3
    :parameters ()
    :precondition (and (b8) (b22) (b30))
    :effect (and (b24) (b26) (b31) (not (b2)) (not (b6)) (not (b28))
      (not (b30))))
\end{lstlisting}
\end{pddlblock}

\begin{pddlblock}
\begin{lstlisting}[style=paperPDDL]
  (:action pick_tomato_soup_can_from_4
    :parameters ()
    :precondition (and (b34) (not (b15)) (not (b16)))
    :effect (and (b16) (b19) (b21) (b24) (b26) (b31) (not (b3)) (not (b4))
      (not (b9)) (not (b34))))
\end{lstlisting}
\end{pddlblock}

\begin{pddlblock}
\begin{lstlisting}[style=paperPDDL]
  (:action pick_tomato_soup_can_from_5
    :parameters ()
    :precondition (and (b22) (not (b2)) (not (b8)))
    :effect (and (b8) (b10) (b16) (b19) (b21) (b24) (b26) (b31)))
\end{lstlisting}
\end{pddlblock}

\begin{pddlblock}
\begin{lstlisting}[style=paperPDDL]
  (:action pick_tomato_soup_can_from_6
    :parameters ()
    :precondition (and (b25) (b32) (not (b16)))
    :effect (and (b8) (b16) (b19) (b24) (b31) (not (b2)) (not (b3))
      (not (b23)) (not (b28)) (not (b30)) (not (b32))))
\end{lstlisting}
\end{pddlblock}

\begin{pddlblock}
\begin{lstlisting}[style=paperPDDL]
  (:action pick_tomato_soup_can_from_7
    :parameters ()
    :precondition (and (b0) (b4) (b11) (b19) (not (b3)) (not (b15)))
    :effect (and (b21) (b24) (b26) (not (b4)) (not (b6))))
\end{lstlisting}
\end{pddlblock}

\begin{pddlblock}
\begin{lstlisting}[style=paperPDDL]
  (:action place_banana_at_1
    :parameters ()
    :precondition (and (b27) (b31) (not (b5)) (not (b26)))
    :effect (and (b3) (b4) (b5) (not (b18)) (not (b19))))
\end{lstlisting}
\end{pddlblock}

\begin{pddlblock}
\begin{lstlisting}[style=paperPDDL]
  (:action place_banana_at_2
    :parameters ()
    :precondition (and (b31) (not (b1)) (not (b4)) (not (b7)) (not (b26)))
    :effect (and (b3) (b13) (b22) (b33) (not (b15)) (not (b18))
      (not (b31))))
\end{lstlisting}
\end{pddlblock}

\begin{pddlblock}
\begin{lstlisting}[style=paperPDDL]
  (:action place_banana_at_3
    :parameters ()
    :precondition (and (b18) (not (b4)) (not (b19)) (not (b22)))
    :effect (and (not (b9)) (not (b15)) (not (b18)) (not (b31))))
\end{lstlisting}
\end{pddlblock}

\begin{pddlblock}
\begin{lstlisting}[style=paperPDDL]
  (:action place_banana_at_4
    :parameters ()
    :precondition (and (b6) (b31) (b32) (not (b12)))
    :effect (and (b12) (b20) (b22) (not (b21)) (not (b28))))
\end{lstlisting}
\end{pddlblock}

\begin{pddlblock}
\begin{lstlisting}[style=paperPDDL]
  (:action place_banana_at_5
    :parameters ()
    :precondition (and (b28) (b31) (not (b13)) (not (b25)))
    :effect (and (b7) (b20) (b22) (b25) (not (b9)) (not (b15)) (not (b31))))
\end{lstlisting}
\end{pddlblock}

\begin{pddlblock}
\begin{lstlisting}[style=paperPDDL]
  (:action place_banana_at_6
    :parameters ()
    :precondition (and (b18) (b31) (not (b2)) (not (b26)))
    :effect (and (b3) (b4) (b11) (b12) (b30) (not (b31))))
\end{lstlisting}
\end{pddlblock}

\begin{pddlblock}
\begin{lstlisting}[style=paperPDDL]
  (:action place_banana_at_7
    :parameters ()
    :precondition (and (b2) (b11) (b18) (not (b22)))
    :effect (and (b4) (b5) (b7) (b8) (b20) (b22) (b24) (b33) (not (b19))
      (not (b31))))
\end{lstlisting}
\end{pddlblock}

\begin{pddlblock}
\begin{lstlisting}[style=paperPDDL]
  (:action place_lemon_at_1
    :parameters ()
    :precondition (and (b7) (b10) (b12) (b18))
    :effect (and (b6) (b17) (b21) (b25) (b28) (not (b12))))
\end{lstlisting}
\end{pddlblock}

\begin{pddlblock}
\begin{lstlisting}[style=paperPDDL]
  (:action place_lemon_at_2
    :parameters ()
    :precondition (and (b10) (b15) (b27))
    :effect (and (not (b10)) (not (b15))))
\end{lstlisting}
\end{pddlblock}

\begin{pddlblock}
\begin{lstlisting}[style=paperPDDL]
  (:action place_lemon_at_3
    :parameters ()
    :precondition (and (b3) (b12) (b13) (b33))
    :effect (and (b0) (b1) (b6) (not (b2)) (not (b3)) (not (b11))
      (not (b12)) (not (b22)) (not (b33))))
\end{lstlisting}
\end{pddlblock}

\begin{pddlblock}
\begin{lstlisting}[style=paperPDDL]
  (:action place_lemon_at_4
    :parameters ()
    :precondition (and (b5) (b6) (b12) (b29) (not (b17)))
    :effect (and (b16) (b17) (not (b12)) (not (b14)) (not (b15))))
\end{lstlisting}
\end{pddlblock}

\begin{pddlblock}
\begin{lstlisting}[style=paperPDDL]
  (:action place_lemon_at_5
    :parameters ()
    :precondition (and (b15) (b27) (not (b23)))
    :effect (and (b20) (b25) (not (b7)) (not (b14)) (not (b15))))
\end{lstlisting}
\end{pddlblock}

\begin{pddlblock}
\begin{lstlisting}[style=paperPDDL]
  (:action place_lemon_at_6
    :parameters ()
    :precondition (and (b7) (b12) (b15) (b20) (b22))
    :effect (and (not (b5)) (not (b7)) (not (b15)) (not (b24))))
\end{lstlisting}
\end{pddlblock}

\begin{pddlblock}
\begin{lstlisting}[style=paperPDDL]
  (:action place_lemon_at_7
    :parameters ()
    :precondition (and (b3) (b11) (b15) (not (b16)) (not (b21)) (not (b30)))
    :effect (and (not (b2)) (not (b11)) (not (b12)) (not (b14))
      (not (b15))))
\end{lstlisting}
\end{pddlblock}

\begin{pddlblock}
\begin{lstlisting}[style=paperPDDL]
  (:action place_mug_at_1
    :parameters ()
    :precondition (and (b7) (b24) (not (b8)) (not (b22)))
    :effect (and (not (b5)) (not (b7)) (not (b19)) (not (b24))))
\end{lstlisting}
\end{pddlblock}

\begin{pddlblock}
\begin{lstlisting}[style=paperPDDL]
  (:action place_mug_at_2
    :parameters ()
    :precondition (and (b9) (b12) (b19) (b25))
    :effect (and (b6) (b14) (b28) (b33) (not (b5)) (not (b24)) (not (b25))))
\end{lstlisting}
\end{pddlblock}

\begin{pddlblock}
\begin{lstlisting}[style=paperPDDL]
  (:action place_mug_at_3
    :parameters ()
    :precondition (and (b5) (b15) (b19) (not (b0)) (not (b16)))
    :effect (and (b10) (b23) (not (b5)) (not (b7)) (not (b9)) (not (b15))
      (not (b24))))
\end{lstlisting}
\end{pddlblock}

\begin{pddlblock}
\begin{lstlisting}[style=paperPDDL]
  (:action place_mug_at_4
    :parameters ()
    :precondition (and (b21) (b24) (b29) (not (b31)))
    :effect (and (b20) (b22) (b25) (not (b19)) (not (b24)) (not (b26))))
\end{lstlisting}
\end{pddlblock}

\begin{pddlblock}
\begin{lstlisting}[style=paperPDDL]
  (:action place_mug_at_5
    :parameters ()
    :precondition (and (b19) (b24) (b26) (not (b16)) (not (b27)))
    :effect (and (b17) (b20) (b25) (not (b19)) (not (b24)) (not (b26))))
\end{lstlisting}
\end{pddlblock}

\begin{pddlblock}
\begin{lstlisting}[style=paperPDDL]
  (:action place_mug_at_6
    :parameters ()
    :precondition (and (b1) (b7) (b24) (not (b23)) (not (b31)))
    :effect (and (b17) (b22) (not (b5)) (not (b19)) (not (b24))))
\end{lstlisting}
\end{pddlblock}

\begin{pddlblock}
\begin{lstlisting}[style=paperPDDL]
  (:action place_mug_at_7
    :parameters ()
    :precondition (and (b13) (b14) (b24) (b28))
    :effect (and (b4) (b18) (not (b2)) (not (b3)) (not (b19)) (not (b21))
      (not (b24)) (not (b28))))
\end{lstlisting}
\end{pddlblock}

\begin{pddlblock}
\begin{lstlisting}[style=paperPDDL]
  (:action place_peach_at_1
    :parameters ()
    :precondition (and (b20) (b24) (b34))
    :effect (and (b7) (not (b0)) (not (b5)) (not (b24)) (not (b34))))
\end{lstlisting}
\end{pddlblock}

\begin{pddlblock}
\begin{lstlisting}[style=paperPDDL]
  (:action place_peach_at_2
    :parameters ()
    :precondition (and (b0) (b10) (b12) (not (b6)) (not (b27)))
    :effect (and (b6) (b13) (b27) (b29) (not (b5)) (not (b10)) (not (b16))
      (not (b24))))
\end{lstlisting}
\end{pddlblock}

\begin{pddlblock}
\begin{lstlisting}[style=paperPDDL]
  (:action place_peach_at_3
    :parameters ()
    :precondition (and (b1) (b32) (not (b6)))
    :effect (and (b4) (b6) (not (b3)) (not (b5))))
\end{lstlisting}
\end{pddlblock}

\begin{pddlblock}
\begin{lstlisting}[style=paperPDDL]
  (:action place_peach_at_4
    :parameters ()
    :precondition (and (b1) (not (b6)) (not (b9)) (not (b17)))
    :effect (and (not (b0)) (not (b1)) (not (b5)) (not (b16)) (not (b24))))
\end{lstlisting}
\end{pddlblock}

\begin{pddlblock}
\begin{lstlisting}[style=paperPDDL]
  (:action place_peach_at_5
    :parameters ()
    :precondition (and (b24) (not (b7)) (not (b14)) (not (b19)) (not (b25))
      (not (b31)))
    :effect (and (b2) (b6) (b12) (b14) (b17) (b25) (not (b16)) (not (b24))))
\end{lstlisting}
\end{pddlblock}

\begin{pddlblock}
\begin{lstlisting}[style=paperPDDL]
  (:action place_peach_at_6
    :parameters ()
    :precondition (and (b34) (not (b17)) (not (b19)))
    :effect (and (b7) (b17) (b27) (not (b16)) (not (b24))))
\end{lstlisting}
\end{pddlblock}

\begin{pddlblock}
\begin{lstlisting}[style=paperPDDL]
  (:action place_peach_at_7
    :parameters ()
    :precondition (and (b3) (b11) (b14) (b24) (not (b15)))
    :effect (and (b7) (b27) (not (b1)) (not (b5)) (not (b8)) (not (b24))
      (not (b34))))
\end{lstlisting}
\end{pddlblock}

\begin{pddlblock}
\begin{lstlisting}[style=paperPDDL]
  (:action place_tomato_soup_can_at_1
    :parameters ()
    :precondition (and (b8) (b13) (not (b3)) (not (b18)))
    :effect (and (b3) (b4) (not (b16)) (not (b19)) (not (b24)) (not (b31))))
\end{lstlisting}
\end{pddlblock}

\begin{pddlblock}
\begin{lstlisting}[style=paperPDDL]
  (:action place_tomato_soup_can_at_2
    :parameters ()
    :precondition (and (b19) (not (b11)) (not (b13)) (not (b30)))
    :effect (and (b28) (b29) (not (b10)) (not (b16)) (not (b19))
      (not (b21)) (not (b24)) (not (b31))))
\end{lstlisting}
\end{pddlblock}

\begin{pddlblock}
\begin{lstlisting}[style=paperPDDL]
  (:action place_tomato_soup_can_at_3
    :parameters ()
    :precondition (and (b19) (b31) (not (b6)) (not (b17)))
    :effect (and (b2) (b6) (b28) (b30) (not (b24)) (not (b26)) (not (b31))))
\end{lstlisting}
\end{pddlblock}

\begin{pddlblock}
\begin{lstlisting}[style=paperPDDL]
  (:action place_tomato_soup_can_at_4
    :parameters ()
    :precondition (and (b17) (b26) (b31) (not (b9)))
    :effect (and (b3) (b4) (b9) (b34) (not (b16)) (not (b19)) (not (b21))
      (not (b24)) (not (b26)) (not (b31))))
\end{lstlisting}
\end{pddlblock}

\begin{pddlblock}
\begin{lstlisting}[style=paperPDDL]
  (:action place_tomato_soup_can_at_5
    :parameters ()
    :precondition (and (b10) (b31) (not (b2)) (not (b20)))
    :effect (and (not (b8)) (not (b10)) (not (b16)) (not (b19))
      (not (b21)) (not (b24)) (not (b26)) (not (b31))))
\end{lstlisting}
\end{pddlblock}

\begin{pddlblock}
\begin{lstlisting}[style=paperPDDL]
  (:action place_tomato_soup_can_at_6
    :parameters ()
    :precondition (and (b24) (b25) (not (b2)))
    :effect (and (b2) (b3) (b23) (b28) (b30) (b32) (not (b8)) (not (b16))
      (not (b19)) (not (b24)) (not (b31))))
\end{lstlisting}
\end{pddlblock}

\begin{pddlblock}
\begin{lstlisting}[style=paperPDDL]
  (:action place_tomato_soup_can_at_7
    :parameters ()
    :precondition (and (b1) (b11) (b21))
    :effect (and (b4) (b6) (not (b21)) (not (b24)) (not (b26))))
)
\end{lstlisting}
\end{pddlblock}

\Needspace{18\baselineskip}
\subsection{DinnerTable (real images)}
\label{app:pddl-dinnertable-real}
\noindent 35 predicates; 64 grounded actions.\par

\begin{pddlblock}
\begin{lstlisting}[style=paperPDDL]
(define (domain dinnertable_real_learned)
  (:requirements :strips :negative-preconditions)
  (:predicates (b0) (b1) (b2) (b3) (b4) (b5) (b6) (b7) (b8) (b9) (b10)
      (b11) (b12) (b13) (b14) (b15) (b16) (b17) (b18) (b19) (b20) (b21)
      (b22) (b23) (b24) (b25) (b26) (b27) (b28) (b29) (b30) (b31) (b32)
      (b33) (b34))
\end{lstlisting}
\end{pddlblock}

\begin{pddlblock}
\begin{lstlisting}[style=paperPDDL]
  (:action pick_banana_from_1
    :parameters ()
    :precondition (and (b5) (not (b3)) (not (b4)) (not (b28)))
    :effect (and (b8) (b9) (b13) (b32) (not (b1)) (not (b5)) (not (b14))
      (not (b18)) (not (b25))))
\end{lstlisting}
\end{pddlblock}

\begin{pddlblock}
\begin{lstlisting}[style=paperPDDL]
  (:action pick_banana_from_2
    :parameters ()
    :precondition (and (b29) (not (b7)) (not (b8)) (not (b11)) (not (b21)))
    :effect (and (b8) (not (b1)) (not (b20)) (not (b29))))
\end{lstlisting}
\end{pddlblock}

\begin{pddlblock}
\begin{lstlisting}[style=paperPDDL]
  (:action pick_banana_from_3
    :parameters ()
    :precondition (and (not (b0)) (not (b2)) (not (b7)))
    :effect (and (b0) (b8) (b12) (b28) (b31)))
\end{lstlisting}
\end{pddlblock}

\begin{pddlblock}
\begin{lstlisting}[style=paperPDDL]
  (:action pick_banana_from_5
    :parameters ()
    :precondition (and (b24) (b31) (not (b26)))
    :effect (and (b8) (not (b16)) (not (b24))))
\end{lstlisting}
\end{pddlblock}

\begin{pddlblock}
\begin{lstlisting}[style=paperPDDL]
  (:action pick_banana_from_6
    :parameters ()
    :precondition (and (b0) (b23) (not (b8)) (not (b22)))
    :effect (and (b8) (not (b7)) (not (b23)) (not (b33))))
\end{lstlisting}
\end{pddlblock}

\begin{pddlblock}
\begin{lstlisting}[style=paperPDDL]
  (:action pick_banana_from_7
    :parameters ()
    :precondition (and (b10) (b32) (not (b15)))
    :effect (and (b8) (b9) (not (b7)) (not (b10)) (not (b17)) (not (b18))
      (not (b24))))
\end{lstlisting}
\end{pddlblock}

\begin{pddlblock}
\begin{lstlisting}[style=paperPDDL]
  (:action pick_lemon_from_1
    :parameters ()
    :precondition (and (b2) (b10) (b25))
    :effect (and (b11) (b15) (not (b23)) (not (b25))))
\end{lstlisting}
\end{pddlblock}

\begin{pddlblock}
\begin{lstlisting}[style=paperPDDL]
  (:action pick_lemon_from_2
    :parameters ()
    :precondition (and (b10) (b11) (b13) (b19))
    :effect (and (b6) (not (b13))))
\end{lstlisting}
\end{pddlblock}

\begin{pddlblock}
\begin{lstlisting}[style=paperPDDL]
  (:action pick_lemon_from_3
    :parameters ()
    :precondition (and (b0) (b7) (b29))
    :effect (and (b2) (b11) (b20) (b31) (not (b5)) (not (b7)) (not (b17))
      (not (b22)) (not (b23))))
\end{lstlisting}
\end{pddlblock}

\begin{pddlblock}
\begin{lstlisting}[style=paperPDDL]
  (:action pick_lemon_from_4
    :parameters ()
    :precondition (and (b11) (b18) (not (b6)))
    :effect (and (b3) (b6) (b15) (b30)))
\end{lstlisting}
\end{pddlblock}

\begin{pddlblock}
\begin{lstlisting}[style=paperPDDL]
  (:action pick_lemon_from_5
    :parameters ()
    :precondition (and (b0) (b19) (b20))
    :effect (and (b4) (b6) (b11) (not (b9)) (not (b13)) (not (b20))))
\end{lstlisting}
\end{pddlblock}

\begin{pddlblock}
\begin{lstlisting}[style=paperPDDL]
  (:action pick_lemon_from_6
    :parameters ()
    :precondition (and (b16) (b18) (b32) (not (b8)))
    :effect (and (b2) (b6) (b10) (b11) (b15) (b30) (not (b16)) (not (b22))
      (not (b26)) (not (b33))))
\end{lstlisting}
\end{pddlblock}

\begin{pddlblock}
\begin{lstlisting}[style=paperPDDL]
  (:action pick_lemon_from_7
    :parameters ()
    :precondition (and (b0) (b1) (not (b2)) (not (b3)) (not (b27)))
    :effect (and (b2) (b3) (b6) (b11)))
\end{lstlisting}
\end{pddlblock}

\begin{pddlblock}
\begin{lstlisting}[style=paperPDDL]
  (:action pick_mug_from_1
    :parameters ()
    :precondition (and (b0) (b7) (b12) (not (b8)) (not (b11)))
    :effect (and (b8) (b15) (not (b0)) (not (b12))))
\end{lstlisting}
\end{pddlblock}

\begin{pddlblock}
\begin{lstlisting}[style=paperPDDL]
  (:action pick_mug_from_2
    :parameters ()
    :precondition (and (b3) (b19) (not (b0)) (not (b6)) (not (b8)))
    :effect (and (b8) (not (b4)) (not (b12)) (not (b14))))
\end{lstlisting}
\end{pddlblock}

\begin{pddlblock}
\begin{lstlisting}[style=paperPDDL]
  (:action pick_mug_from_3
    :parameters ()
    :precondition (and (b9) (b30) (not (b0)) (not (b8)))
    :effect (and (b8) (b28) (b31) (not (b9)) (not (b30)) (not (b32))))
\end{lstlisting}
\end{pddlblock}

\begin{pddlblock}
\begin{lstlisting}[style=paperPDDL]
  (:action pick_mug_from_5
    :parameters ()
    :precondition (and (b6) (b9) (b24))
    :effect (and (b8) (b28) (b29) (not (b2)) (not (b24))))
\end{lstlisting}
\end{pddlblock}

\begin{pddlblock}
\begin{lstlisting}[style=paperPDDL]
  (:action pick_mug_from_6
    :parameters ()
    :precondition (and (b10) (not (b0)) (not (b3)) (not (b4)))
    :effect (and (b3) (b8) (not (b10))))
\end{lstlisting}
\end{pddlblock}

\begin{pddlblock}
\begin{lstlisting}[style=paperPDDL]
  (:action pick_mug_from_7
    :parameters ()
    :precondition (and (b8) (b17) (b30) (not (b21)))
    :effect (and (b3) (not (b0)) (not (b15)) (not (b17)) (not (b30))))
\end{lstlisting}
\end{pddlblock}

\begin{pddlblock}
\begin{lstlisting}[style=paperPDDL]
  (:action pick_peach_from_1
    :parameters ()
    :precondition (and (b12) (b13) (not (b10)))
    :effect (and (b8) (b15) (b20) (not (b12)) (not (b17))))
\end{lstlisting}
\end{pddlblock}

\begin{pddlblock}
\begin{lstlisting}[style=paperPDDL]
  (:action pick_peach_from_2
    :parameters ()
    :precondition (and (b13) (b33) (not (b11)))
    :effect (and (b8) (b14) (not (b13))))
\end{lstlisting}
\end{pddlblock}

\begin{pddlblock}
\begin{lstlisting}[style=paperPDDL]
  (:action pick_peach_from_4
    :parameters ()
    :precondition (and (b33) (not (b0)) (not (b8)))
    :effect (and (b0) (b8) (b14)))
\end{lstlisting}
\end{pddlblock}

\begin{pddlblock}
\begin{lstlisting}[style=paperPDDL]
  (:action pick_peach_from_5
    :parameters ()
    :precondition (and (b7) (b19) (b24))
    :effect (and (b8) (not (b7)) (not (b11)) (not (b13)) (not (b24))))
\end{lstlisting}
\end{pddlblock}

\begin{pddlblock}
\begin{lstlisting}[style=paperPDDL]
  (:action pick_peach_from_6
    :parameters ()
    :precondition (and (b22) (not (b6)) (not (b9)) (not (b26)))
    :effect (and (b8) (not (b1)) (not (b5)) (not (b10)) (not (b12))
      (not (b22)) (not (b24))))
\end{lstlisting}
\end{pddlblock}

\begin{pddlblock}
\begin{lstlisting}[style=paperPDDL]
  (:action pick_peach_from_7
    :parameters ()
    :precondition (and (b3) (b4) (b7) (b13))
    :effect (and (b2) (b8) (not (b7)) (not (b10)) (not (b12)) (not (b25))))
\end{lstlisting}
\end{pddlblock}

\begin{pddlblock}
\begin{lstlisting}[style=paperPDDL]
  (:action pick_tomato_soup_can_from_1
    :parameters ()
    :precondition (and (b5) (b6) (b9) (b21) (not (b8)))
    :effect (and (b4) (b15) (b27) (not (b9)) (not (b20))))
\end{lstlisting}
\end{pddlblock}

\begin{pddlblock}
\begin{lstlisting}[style=paperPDDL]
  (:action pick_tomato_soup_can_from_2
    :parameters ()
    :precondition (and (b19) (b24) (b27))
    :effect (and (b26) (not (b19)) (not (b33))))
\end{lstlisting}
\end{pddlblock}

\begin{pddlblock}
\begin{lstlisting}[style=paperPDDL]
  (:action pick_tomato_soup_can_from_3
    :parameters ()
    :precondition (and (b4) (not (b8)) (not (b19)) (not (b26)))
    :effect (and (b21) (b26) (b28) (not (b5))))
\end{lstlisting}
\end{pddlblock}

\begin{pddlblock}
\begin{lstlisting}[style=paperPDDL]
  (:action pick_tomato_soup_can_from_4
    :parameters ()
    :precondition (and (b13) (b16) (b19) (b24))
    :effect (and (b14) (b21) (b26) (b27) (not (b16)) (not (b19))))
\end{lstlisting}
\end{pddlblock}

\begin{pddlblock}
\begin{lstlisting}[style=paperPDDL]
  (:action pick_tomato_soup_can_from_5
    :parameters ()
    :precondition (and (b28) (b29) (not (b3)) (not (b26)))
    :effect (and (b4) (b21) (b26) (b27)))
\end{lstlisting}
\end{pddlblock}

\begin{pddlblock}
\begin{lstlisting}[style=paperPDDL]
  (:action pick_tomato_soup_can_from_6
    :parameters ()
    :precondition (and (b2) (b16) (b23))
    :effect (and (b1) (b15) (b17) (b21) (not (b9)) (not (b10)) (not (b16))
      (not (b18))))
\end{lstlisting}
\end{pddlblock}

\begin{pddlblock}
\begin{lstlisting}[style=paperPDDL]
  (:action pick_tomato_soup_can_from_7
    :parameters ()
    :precondition (and (b1) (b13) (b19) (b21))
    :effect (and (b5) (b27) (not (b19))))
\end{lstlisting}
\end{pddlblock}

\begin{pddlblock}
\begin{lstlisting}[style=paperPDDL]
  (:action place_banana_at_1
    :parameters ()
    :precondition (and (b0) (b8) (b9) (b13))
    :effect (and (b1) (b5) (b14) (b18) (b25) (not (b8)) (not (b9))
      (not (b13)) (not (b32))))
\end{lstlisting}
\end{pddlblock}

\begin{pddlblock}
\begin{lstlisting}[style=paperPDDL]
  (:action place_banana_at_2
    :parameters ()
    :precondition (and (b8) (not (b13)) (not (b19)) (not (b29)) (not (b33)))
    :effect (and (b1) (b20) (b29) (not (b8))))
\end{lstlisting}
\end{pddlblock}

\begin{pddlblock}
\begin{lstlisting}[style=paperPDDL]
  (:action place_banana_at_3
    :parameters ()
    :precondition (and (b8) (b12) (not (b5)))
    :effect (and (not (b0)) (not (b8)) (not (b12)) (not (b28)) (not (b31))))
\end{lstlisting}
\end{pddlblock}

\begin{pddlblock}
\begin{lstlisting}[style=paperPDDL]
  (:action place_banana_at_5
    :parameters ()
    :precondition (and (b4) (b8) (b12))
    :effect (and (b16) (b24) (not (b8))))
\end{lstlisting}
\end{pddlblock}

\begin{pddlblock}
\begin{lstlisting}[style=paperPDDL]
  (:action place_banana_at_6
    :parameters ()
    :precondition (and (b8) (not (b2)) (not (b4)) (not (b7)))
    :effect (and (b7) (b23) (b33) (not (b8))))
\end{lstlisting}
\end{pddlblock}

\begin{pddlblock}
\begin{lstlisting}[style=paperPDDL]
  (:action place_banana_at_7
    :parameters ()
    :precondition (and (b2) (b8) (b32))
    :effect (and (b7) (b10) (b17) (b18) (b24) (not (b8)) (not (b9))))
\end{lstlisting}
\end{pddlblock}

\begin{pddlblock}
\begin{lstlisting}[style=paperPDDL]
  (:action place_lemon_at_1
    :parameters ()
    :precondition (and (b2) (b10) (b15) (not (b13)))
    :effect (and (b23) (b25) (not (b11)) (not (b15))))
\end{lstlisting}
\end{pddlblock}

\begin{pddlblock}
\begin{lstlisting}[style=paperPDDL]
  (:action place_lemon_at_2
    :parameters ()
    :precondition (and (b3) (b4) (b6) (b11))
    :effect (and (b13) (not (b6))))
\end{lstlisting}
\end{pddlblock}

\begin{pddlblock}
\begin{lstlisting}[style=paperPDDL]
  (:action place_lemon_at_3
    :parameters ()
    :precondition (and (b11) (b31))
    :effect (and (b5) (b7) (b17) (b22) (b23) (not (b2)) (not (b11))
      (not (b20)) (not (b31))))
\end{lstlisting}
\end{pddlblock}

\begin{pddlblock}
\begin{lstlisting}[style=paperPDDL]
  (:action place_lemon_at_4
    :parameters ()
    :precondition (and (b0) (b3) (b15))
    :effect (and (not (b3)) (not (b6)) (not (b15)) (not (b30))))
\end{lstlisting}
\end{pddlblock}

\begin{pddlblock}
\begin{lstlisting}[style=paperPDDL]
  (:action place_lemon_at_5
    :parameters ()
    :precondition (and (b6) (b19))
    :effect (and (b9) (b13) (b20) (not (b4)) (not (b6)) (not (b11))))
\end{lstlisting}
\end{pddlblock}

\begin{pddlblock}
\begin{lstlisting}[style=paperPDDL]
  (:action place_lemon_at_6
    :parameters ()
    :precondition (and (b2) (b3) (b10) (b15))
    :effect (and (b16) (b22) (b26) (b33) (not (b2)) (not (b6)) (not (b10))
      (not (b11)) (not (b15)) (not (b30))))
\end{lstlisting}
\end{pddlblock}

\begin{pddlblock}
\begin{lstlisting}[style=paperPDDL]
  (:action place_lemon_at_7
    :parameters ()
    :precondition (and (b0) (b1) (b2) (b3) (b6) (not (b26)))
    :effect (and (not (b2)) (not (b3)) (not (b6)) (not (b11))))
\end{lstlisting}
\end{pddlblock}

\begin{pddlblock}
\begin{lstlisting}[style=paperPDDL]
  (:action place_mug_at_1
    :parameters ()
    :precondition (and (b7) (b8) (b15) (not (b0)))
    :effect (and (b0) (b12) (not (b8)) (not (b15))))
\end{lstlisting}
\end{pddlblock}

\begin{pddlblock}
\begin{lstlisting}[style=paperPDDL]
  (:action place_mug_at_2
    :parameters ()
    :precondition (and (b3) (b8) (not (b0)) (not (b1)) (not (b6))
      (not (b14)))
    :effect (and (b4) (b12) (b14) (not (b8))))
\end{lstlisting}
\end{pddlblock}

\begin{pddlblock}
\begin{lstlisting}[style=paperPDDL]
  (:action place_mug_at_3
    :parameters ()
    :precondition (and (b31) (not (b0)))
    :effect (and (b9) (b30) (b32) (not (b8)) (not (b28)) (not (b31))))
\end{lstlisting}
\end{pddlblock}

\begin{pddlblock}
\begin{lstlisting}[style=paperPDDL]
  (:action place_mug_at_5
    :parameters ()
    :precondition (and (b3) (b5) (b8) (b26))
    :effect (and (b2) (b24) (not (b8)) (not (b28)) (not (b29))))
\end{lstlisting}
\end{pddlblock}

\begin{pddlblock}
\begin{lstlisting}[style=paperPDDL]
  (:action place_mug_at_6
    :parameters ()
    :precondition (and (b8) (b17) (not (b0)))
    :effect (and (b10) (not (b3)) (not (b8))))
\end{lstlisting}
\end{pddlblock}

\begin{pddlblock}
\begin{lstlisting}[style=paperPDDL]
  (:action place_mug_at_7
    :parameters ()
    :precondition (and (b3) (b8) (b29) (not (b17)))
    :effect (and (b0) (b15) (b17) (b30) (not (b3))))
\end{lstlisting}
\end{pddlblock}

\begin{pddlblock}
\begin{lstlisting}[style=paperPDDL]
  (:action place_peach_at_1
    :parameters ()
    :precondition (and (b2) (b4) (b8) (b14))
    :effect (and (b12) (b17) (not (b8)) (not (b15)) (not (b20))))
\end{lstlisting}
\end{pddlblock}

\begin{pddlblock}
\begin{lstlisting}[style=paperPDDL]
  (:action place_peach_at_2
    :parameters ()
    :precondition (and (b14) (b32))
    :effect (and (b13) (not (b8)) (not (b14))))
\end{lstlisting}
\end{pddlblock}

\begin{pddlblock}
\begin{lstlisting}[style=paperPDDL]
  (:action place_peach_at_4
    :parameters ()
    :precondition (and (b14) (not (b1)) (not (b13)))
    :effect (and (not (b0)) (not (b8)) (not (b14))))
\end{lstlisting}
\end{pddlblock}

\begin{pddlblock}
\begin{lstlisting}[style=paperPDDL]
  (:action place_peach_at_5
    :parameters ()
    :precondition (and (b0) (b19) (not (b11)) (not (b12)))
    :effect (and (b7) (b11) (b13) (b24) (not (b8))))
\end{lstlisting}
\end{pddlblock}

\begin{pddlblock}
\begin{lstlisting}[style=paperPDDL]
  (:action place_peach_at_6
    :parameters ()
    :precondition (and (b4) (b8) (not (b15)) (not (b27)))
    :effect (and (b1) (b5) (b10) (b12) (b22) (b24) (not (b8))))
\end{lstlisting}
\end{pddlblock}

\begin{pddlblock}
\begin{lstlisting}[style=paperPDDL]
  (:action place_peach_at_7
    :parameters ()
    :precondition (and (b2) (b4) (b8) (b13))
    :effect (and (b7) (b10) (b12) (b25) (not (b2)) (not (b8))))
\end{lstlisting}
\end{pddlblock}

\begin{pddlblock}
\begin{lstlisting}[style=paperPDDL]
  (:action place_tomato_soup_can_at_1
    :parameters ()
    :precondition (and (b17) (b27) (not (b30)))
    :effect (and (b9) (b20) (not (b4)) (not (b15)) (not (b27))))
\end{lstlisting}
\end{pddlblock}

\begin{pddlblock}
\begin{lstlisting}[style=paperPDDL]
  (:action place_tomato_soup_can_at_2
    :parameters ()
    :precondition (and (b0) (b21) (b26) (not (b1)))
    :effect (and (b19) (b33) (not (b26))))
\end{lstlisting}
\end{pddlblock}

\begin{pddlblock}
\begin{lstlisting}[style=paperPDDL]
  (:action place_tomato_soup_can_at_3
    :parameters ()
    :precondition (and (b12) (b26) (b28))
    :effect (and (b5) (not (b21)) (not (b26)) (not (b28))))
\end{lstlisting}
\end{pddlblock}

\begin{pddlblock}
\begin{lstlisting}[style=paperPDDL]
  (:action place_tomato_soup_can_at_4
    :parameters ()
    :precondition (and (b13) (b14) (b26) (b27))
    :effect (and (b16) (b19) (not (b14)) (not (b21)) (not (b26))
      (not (b27))))
\end{lstlisting}
\end{pddlblock}

\begin{pddlblock}
\begin{lstlisting}[style=paperPDDL]
  (:action place_tomato_soup_can_at_5
    :parameters ()
    :precondition (and (b1) (b4) (b28))
    :effect (and (not (b4)) (not (b21)) (not (b26)) (not (b27))))
\end{lstlisting}
\end{pddlblock}

\begin{pddlblock}
\begin{lstlisting}[style=paperPDDL]
  (:action place_tomato_soup_can_at_6
    :parameters ()
    :precondition (and (b2) (b15) (b21))
    :effect (and (b9) (b10) (b16) (b18) (not (b1)) (not (b15)) (not (b17))
      (not (b21))))
\end{lstlisting}
\end{pddlblock}

\begin{pddlblock}
\begin{lstlisting}[style=paperPDDL]
  (:action place_tomato_soup_can_at_7
    :parameters ()
    :precondition (and (b5) (b10) (b14))
    :effect (and (b19) (not (b5)) (not (b27))))
)
\end{lstlisting}
\end{pddlblock}

\end{document}